\documentclass[runningheads]{llncs}

\usepackage[T1]{fontenc}
\usepackage{graphicx}

\usepackage{amsmath}
\usepackage{amsfonts}
\usepackage{booktabs}
\usepackage{multirow}
\usepackage{xcolor}
\usepackage[unicode=true,hidelinks]{hyperref}
\usepackage{algorithm}
\usepackage{algorithmic}
\usepackage{orcidlink}
\usepackage{float}

\newcommand{\TV}{\mathcal{T}}
\newcommand{\TRUE}{\textsc{True}}
\newcommand{\UND}{\textsc{Undec}}
\newcommand{\FALSE}{\textsc{False}}
\newcommand{\delt}{\delta}
\newcommand{\thet}{\theta}
\newcommand{\decacc}{\text{Dec-Acc}}

\begin{document}
\setlength{\emergencystretch}{1.5em}

\title{Ternary Decision Trees with Locally-Adaptive Uncertainty Zones}

\author{William Smits \orcidlink{0009-0007-1673-8172}}
\institute{Avathon, Austin, TX, USA \\
\email{will.smits@avathongov.com}}
\maketitle

\begin{abstract}
Decision trees partition the feature space using hard binary thresholds,
assigning identical confidence to instances far from a decision boundary
and to those directly on it.
We introduce \emph{ternary decision trees}, which augment each split node
with an \emph{uncertainty zone} of half-width $\delta$ centred on the
optimal threshold.
Instances whose feature value falls within this zone receive predictions
formed by weighted blending of both child subtrees and are flagged as
\emph{boundary-uncertain}, signalling that downstream applications may
wish to treat these predictions differently.
Crucially, $\delta$ is computed \emph{locally at each node} from
statistics already available during standard CART split finding,
requiring no external noise specification and no additional data.
We ground this approach in a decision-theoretic framework that
characterises the optimal zone half-width $\delta^*$ at each node as
the solution to a local cost-minimisation problem, and establish four
formal properties: an accuracy decomposition, a sufficiency condition
for decided accuracy improvement, an exact efficiency characterisation
($\eta = \mathrm{Dec\text{-}Acc} - \mathrm{Acc}_u$, the accuracy gap
between decided and boundary-uncertain predictions), and asymptotic
consistency of the margin method.
We propose and evaluate five $\delta$-estimation methods:
\emph{quality-plateau} (plateau width of the split criterion curve),
\emph{class-overlap} (empirical class-distribution overlap),
\emph{gain-ratio} (split quality relative to split entropy),
\emph{node-bootstrap} (threshold variance under node-level resampling),
and \emph{margin} (SVM-inspired distance to the nearest cross-class
training example).
Evaluated across 71 of the 72 OpenML-CC18 datasets with 5-fold cross-validation
(one image dataset with $d=3{,}072$ raw pixel features was excluded),
all five methods with probabilistic routing significantly outperform
standard CART on decided accuracy (Wilcoxon signed-rank, $p \le 0.001$).
The margin method achieves the best efficiency (0.104 accuracy gain per
unit of boundary-uncertain flagging rate), wins on 42 of 72 datasets,
and requires zero additional hyperparameters.
Analysis on three Breiman synthetic benchmarks with known Bayes errors
reveals that the margin method is self-calibrating on geometrically clean
data, while node-bootstrap and quality-plateau maintain the best ratio of
flagging rate to theoretical irreducible error.
Experiments on four medical and financial datasets demonstrate practical
value: on mammography, node-bootstrap achieves $+0.71\%$ decided accuracy
by flagging $10.8\%$ of screening cases as boundary-uncertain.
\end{abstract}

\section{Introduction}
\label{sec:intro}

Decision trees are among the most widely used classification algorithms
due to their interpretability, computational efficiency, and competitive
accuracy on tabular data~\cite{breiman1984classification,quinlan1993c4}.
At each internal node, CART selects the feature $f$ and threshold
$\thet$ that maximise a split criterion (Gini impurity or information
gain), routing instances with $x_f \le \thet$ left and $x_f > \thet$
right.
This mechanism is fundamentally \emph{binary}: every instance is routed
to exactly one child, and the confidence of that routing is the same
regardless of whether the instance sits far from $\thet$ or within a
single unit of it.

This hard commitment creates a known weakness near decision boundaries.
An instance with $x_f = \thet + \varepsilon$ is treated identically to
one with $x_f = \thet + 10\varepsilon$, yet the former sits in a region
where the optimal threshold may genuinely be uncertain given the training
data.
Classifying near-boundary instances with the same confidence as
clearly-classified instances overstates certainty in precisely the cases
where the learned split is least reliable.

We address this by introducing a \emph{ternary} evaluation at each
decision node.
Around the optimal threshold $\thet$, we define an
\emph{uncertainty zone} of half-width $\delt$:
\begin{equation}
\text{zone}(x_f) = \begin{cases}
  \TRUE  & \text{if } x_f > \thet + \delt \\
  \UND   & \text{if } \thet - \delt < x_f \le \thet + \delt \\
  \FALSE & \text{if } x_f \le \thet - \delt
\end{cases}
\label{eq:zone}
\end{equation}
Instances in the \TRUE{} and \FALSE{} zones are routed deterministically.
Instances in the \UND{} zone receive predictions formed by
distance-weighted blending of both child subtrees and are flagged as
\emph{boundary-uncertain}, enabling downstream systems to treat them
differently from confident predictions.
Figure~\ref{fig:decision_surface} illustrates the resulting three-region
partition on a two-dimensional synthetic dataset.

\begin{figure}[t]
\centering
\includegraphics[width=\textwidth]{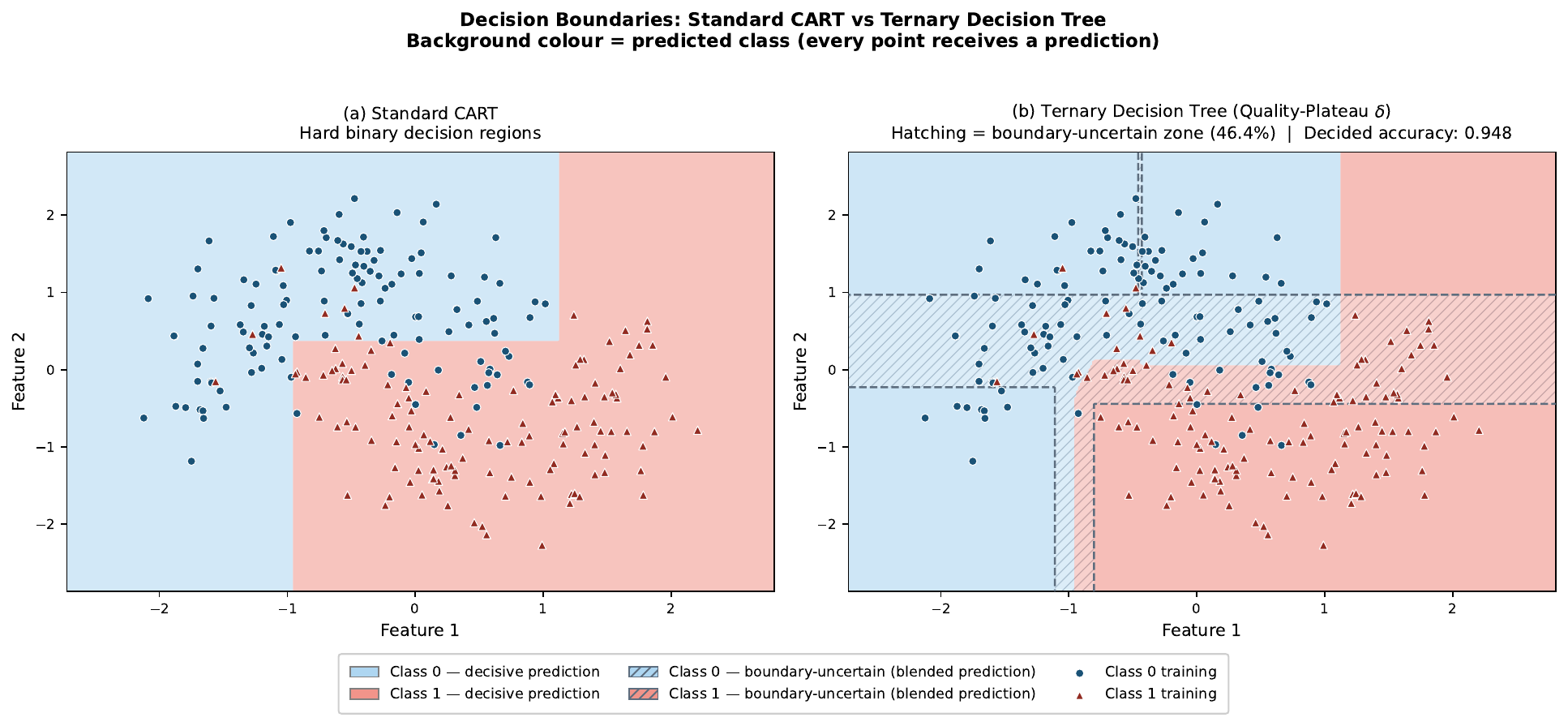}
\caption{Decision boundaries on a two-moons synthetic dataset
(quality-plateau $\delta$ method, depth~3).
Standard CART (left) makes hard binary decisions at every point.
The ternary decision tree (right) shows the same two-class colour
background everywhere --- every point receives a class prediction.
The hatched overlay marks the \emph{boundary-uncertain zone} where
that prediction is formed by weighted blending of both child subtree
outputs rather than deterministic routing.
The hatching qualifies the \emph{confidence} of the prediction,
not its presence: the background colour reveals what class is predicted
even in the uncertain zone.
Training instances are shown as scatter points.
(Blue\,+\,hatch = Class~0 boundary-uncertain;
Red\,+\,hatch = Class~1 boundary-uncertain.)}
\label{fig:decision_surface}
\end{figure}

The central technical challenge is computing $\delt$ appropriately for
each node.
The optimal $\delt^*$ at a given node is the distance from $\thet^*$
within which the expected cost of a decisive prediction exceeds the cost
of a boundary-uncertain one --- a quantity that depends on the local class
geometry and can differ substantially between nodes at the same depth
(Section~\ref{sec:motivation}).
Prior work either specifies $\delt$ externally from a known noise
model~\cite{kent2022indecision}, which is inapplicable when no such model
is available, or derives global thresholds from dataset-level Bayesian risk
parameters~\cite{yao2010threeway,zhi2022fuzzy3wd}, assigning the same
boundary region to every split node regardless of how confident or
ambiguous the local split is.
We derive $\delt$ \emph{locally at each node} from quantities already
computed during standard CART split finding, requiring no additional
data and no external noise model.

\paragraph{Contributions.}
\begin{enumerate}
\item Five methods for computing a node-local uncertainty zone
      half-width $\delt$ from split statistics.
      A decision-theoretic framework establishes the theoretically
      optimal half-width $\delt^*$ as the solution to a node-local
      cost-minimisation problem; the five methods are tractable
      approximations to this optimum, each estimating a different
      observable property of the local error probability $p_e(d)$
      (Sections~\ref{sec:method} and~\ref{sec:motivation}).
\item Four formal theoretical properties: an accuracy decomposition,
      a sufficiency condition for decided accuracy improvement, an
      exact characterisation of the efficiency metric, and asymptotic
      consistency of the margin method
      (Section~\ref{sec:theory}).
\item Two routing architectures for propagating boundary-uncertain
      instances through the tree
      (Section~\ref{sec:routing}).
\item An extensive empirical evaluation across 71 of the 72 OpenML-CC18
      datasets (one high-dimensional image dataset excluded), three Breiman
      synthetic benchmarks with known Bayes errors, and four
      medical/financial datasets (Section~\ref{sec:experiments}).
\item A diagnostic framework using the Undecided/Bayes ratio to
      assess the sensitivity of each $\delt$ method relative to the
      theoretical irreducible error of a dataset
      (Section~\ref{sec:breiman}).
\end{enumerate}

\section{Related Work}
\label{sec:related}

\paragraph{Three-Way Decisions and Rough Sets.}
Yao's three-way decisions framework~\cite{yao2010threeway,yao2012twdsurvey}
partitions instances into positive, boundary, and negative regions
corresponding to acceptance, deferral, and rejection.
The thresholds defining these regions are derived from global Bayesian
risk functions using dataset-level loss parameters rather than local node
statistics.
Pawlak's rough set theory~\cite{pawlak1982roughsets} formalises a
three-region structure identical in form to Equation~\ref{eq:zone},
and the formal connection to three-valued (Kleene) logic has been
established by Avron and Konikowska~\cite{avron2008roughkleene}.
These frameworks provide the theoretical grounding for our approach;
our contribution is a computationally tractable method for deriving
the region boundaries locally at each tree node.

\paragraph{Indecision Trees.}
Kent and M\'enager~\cite{kent2022indecision} modify decision trees to
propagate externally specified measurement uncertainty through the
tree structure.
Their uncertainty zones are derived from known sensor noise rather than
from the training data itself.
Our work addresses the more common setting where no external uncertainty
model is available, deriving $\delt$ entirely from the training data at
each node.

\paragraph{Fuzzy and Soft Decision Trees.}
Fuzzy decision trees~\cite{olaru2003fuzzy} use graded membership
functions at each node, providing a continuous analogue of our discrete
three-zone structure.
Zhi et al.~\cite{zhi2022fuzzy3wd} combine fuzzy trees with three-way
decisions but apply a global post-hoc threshold rather than per-node
local adaptation.
Soft decision trees~\cite{irsoy2012softdt,frosst2017softdt} use learned
sigmoid routing functions, optimised by gradient descent, producing
probabilistic routing without an explicit uncertainty zone.
None of these approaches derives a per-node uncertainty zone from the
local split quality statistics.

\paragraph{Abstaining and Reject-Option Classifiers.}
A related line of work adds a reject option to binary
classifiers~\cite{chow1970optimum,bartlett2008abstaining}.
The reject option is applied at the classifier output level and does
not modify the internal tree structure.
Credal classifiers~\cite{zaffalon2002credal} produce set-valued
predictions using imprecise probability, providing uncertainty at the
output level without structural changes to the model.
In our framework, \emph{boundary-uncertain} instances still receive a
class prediction (formed by weighted subtree blending), so the
distinction from abstention is fundamental: we never withhold a
prediction; we qualify it.

\paragraph{Dominance-Based Rough Sets.}
The Dominance-Based Rough Set Approach (DRSA)~\cite{greco2001drsa}
uses ordered comparison operators $(\ge, \le)$ to produce certain and
possible rules analogous to our \TRUE{} and \UND{} zones.
However, in DRSA the three-zone structure is a property of rule coverage
across the training set rather than of individual predicate evaluation.
Our framework produces a three-valued evaluation \emph{per instance per
predicate}, which is structurally different.

\paragraph{Limitations of existing approaches.}
The methods surveyed above share a fundamental gap: none derives the
uncertainty zone half-width locally from the quality of the individual
split, and none provides an optimality argument for the chosen zone width.
Three-way decisions use dataset-level loss parameters to set the boundary
region globally; every node in every tree receives the same thresholds
regardless of local split confidence~\cite{yao2010threeway}.
Indecision Trees adapt to individual nodes but require an externally
specified noise model that is unavailable in most classification
settings~\cite{kent2022indecision}.
Fuzzy methods use globally optimised membership parameters that cannot
distinguish a node with a wide, ambiguous quality plateau from one with
a sharp, decisive split~\cite{olaru2003fuzzy,zhi2022fuzzy3wd}.
Soft decision trees learn routing probabilities by gradient descent,
producing no interpretable uncertainty zone and no connection to split
quality statistics~\cite{irsoy2012softdt}.
Abstaining and credal classifiers operate at the model output level and
cannot localise uncertainty to specific decision boundaries within the
tree~\cite{chow1970optimum,zaffalon2002credal}.
Imprecise probability approaches to decision
trees~\cite{zaffalon2002credal} do provide node-level uncertainty
quantification via set-valued probability intervals, but derive zone
widths from imprecision parameters rather than from local split
quality statistics, and do not adapt to threshold stability.
Our decision-theoretic framework (Section~\ref{sec:motivation})
addresses all of these limitations: the optimal zone half-width $\delt^*$
is characterised as the solution to a node-local cost-minimisation problem
(Equation~\ref{eq:delta_optimal}), and the five estimation methods serve
as tractable approximations under different modelling assumptions
(Table~\ref{tab:estimators}); all five are derived from statistics
already computed during standard CART split evaluation, requiring no
separate optimization pass and no additional data access.

\section{Ternary Decision Trees}
\label{sec:method}

\subsection{Notation and Preliminaries}

Let $\mathcal{D} = \{(\mathbf{x}_i, y_i)\}_{i=1}^N$ be a training set
with $\mathbf{x}_i \in \mathbb{R}^d$ and $y_i \in \{0, 1, \ldots, K-1\}$.
A standard CART node selects $f^* \in \{1,\ldots,d\}$ and
$\thet^* \in \mathbb{R}$ maximising the split quality criterion $Q$
(Gini gain or information gain) over the $n$ training instances
reaching that node.

We denote the ternary value set as $\TV = \{\TRUE, \UND, \FALSE\}$
with numeric encoding $\TRUE=1$, $\UND=0$, $\FALSE=-1$.
Ternary conjunction follows Kleene's strong three-valued logic:
$\text{AND}(a, b) = \min(a, b)$, i.e.,
\begin{equation}
\TRUE \wedge \UND = \UND, \quad \FALSE \wedge \UND = \FALSE,
\quad \TRUE \wedge \TRUE = \TRUE.
\label{eq:and}
\end{equation}

\subsection{Uncertainty Zone Semantics}

The uncertainty zone of half-width $\delt$ around $\thet^*$ represents
\emph{threshold uncertainty}: the observation that the optimal threshold
is not perfectly determined by the finite training set, and values close
to $\thet^*$ might lie on either side of the true boundary.
This interpretation motivates the five estimation methods below.
Setting $\delt = 0$ recovers standard binary CART behaviour.

\subsection{Decision-Theoretic Motivation}
\label{sec:motivation}

We derive the uncertainty zone half-width $\delt^*$ from a formal
cost-sensitive framework, unifying the five estimation methods of
Section~\ref{sec:deltamethods} under a single principle.

\paragraph{Cost structure.}
At each split node we associate three costs with the three possible
verdicts for an arriving instance:
\begin{itemize}
\item $c_c \ge 0$: cost of a decisive and \emph{correct} prediction;
\item $c_w > c_c$: cost of a decisive and \emph{wrong} prediction;
\item $c_u \in (c_c, c_w)$: cost of a \emph{boundary-uncertain}
      prediction, which still produces a blended class output.
\end{itemize}
Setting $c_c = 0$ without loss of generality, the cost ratio
$\lambda = c_u / c_w \in (0, 1)$ captures the relative cost of
flagging versus decisive misclassification.
High-stakes domains (medical, financial) have $c_u \ll c_w$, giving
small $\lambda$ and large uncertainty zones; high-throughput, low-stakes
settings have $c_u \approx c_w$, giving $\lambda \approx 1$ and
recovering near-standard CART behaviour.

\paragraph{Local error function.}
Let $d = |x_f - \thet^*|$ denote the distance of a feature value
from the optimal threshold $\thet^*$.
Define the \emph{local error probability}
\begin{equation}
p_e(d) \;=\; \Pr\!\left[\hat{y} \ne y \;\middle|\;
             |x_f - \thet^*| = d\right],
\label{eq:pe}
\end{equation}
the probability that a decisive prediction for an instance at
distance $d$ from $\thet^*$ is incorrect.
Intuitively $p_e$ is non-increasing in $d$: instances far from the
threshold are easy to classify correctly, while instances near
$\thet^*$ straddle the decision boundary.
This monotonicity holds when the class-conditional densities of
$x_f$ satisfy the monotone likelihood ratio property near $\thet^*$,
which is satisfied by exponential family distributions and holds
approximately whenever $\thet^*$ is a meaningful discriminative
threshold.
At $d = 0$, $p_e(0)$ is maximised (approaching $0.5$ for balanced,
overlapping classes); as $d \to \infty$, $p_e(d) \to 0$.

\paragraph{Optimal uncertainty zone.}
The expected cost of a decisive verdict at distance $d$ is
$p_e(d) \cdot c_w$.
The expected cost of a boundary-uncertain verdict is $c_u$,
independent of $d$ (the blending mechanism absorbs the uncertainty
regardless of how close to $\thet^*$ the instance falls).
The rational verdict is boundary-uncertain whenever the decisive
expected cost exceeds the uncertain cost:
\begin{equation}
p_e(d) \cdot c_w \;>\; c_u
\;\iff\;
p_e(d) \;>\; \lambda.
\label{eq:threshold_condition}
\end{equation}
This defines the \emph{optimal uncertainty zone half-width} as:
\begin{equation}
\delt^* \;=\; \sup\!\left\{d \ge 0 \;:\; p_e(d) > \lambda\right\}.
\label{eq:delta_optimal}
\end{equation}
Under mild regularity conditions on $p_e$ (monotone non-increasing,
right-continuous), $\delt^*$ is well-defined and satisfies
$p_e(\delt^*) = \lambda$.
Setting $\delt = \delt^*$ at each node recovers the minimum expected
cost over all threshold-based routing policies.

\paragraph{The five methods as estimators of $p_e$.}
In practice $p_e(d)$ is not directly observable from finite training
data.
The five methods of Section~\ref{sec:deltamethods} each provide a
tractable approximation to $\delt^*$ by estimating a different
property of $p_e$, as summarised in Table~\ref{tab:estimators} and
illustrated in Figure~\ref{fig:motivating}.

\begin{table}[t]
\caption{The five $\delt$ methods as estimators of the local error
function $p_e(d)$.  Each approximates $\delt^*$ under a different
modelling assumption about how $p_e$ relates to observable node
statistics.}
\label{tab:estimators}
\centering
\small
\begin{tabular}{lll}
\toprule
Method & Estimated quantity & Assumption \\
\midrule
Node-bootstrap   & Variance of $\thet^*$ under resampling
                 & $\mathrm{std}(\thet^*_b) \propto$ scale of $p_e$ \\
Margin           & Distance to nearest cross-class example
                 & Geometric lower bound on support of $p_e$ \\
Quality-plateau  & Width of near-optimal criterion plateau
                 & Plateau width $\propto$ region where $p_e > \lambda$ \\
Gain-ratio       & Inverse of normalised information gain
                 & Low GR $\Rightarrow$ high avg.\ $p_e$ at this node;
                   indirect proxy (does not model $p_e$ as a function of $d$) \\
Class-overlap    & Empirical overlap of class distributions
                 & Marginal overlap as indirect proxy for the
                   region where $p_e(d)$ may be non-negligible \\
\bottomrule
\end{tabular}
\end{table}

\begin{figure}[t]
\centering
\includegraphics[width=\textwidth]{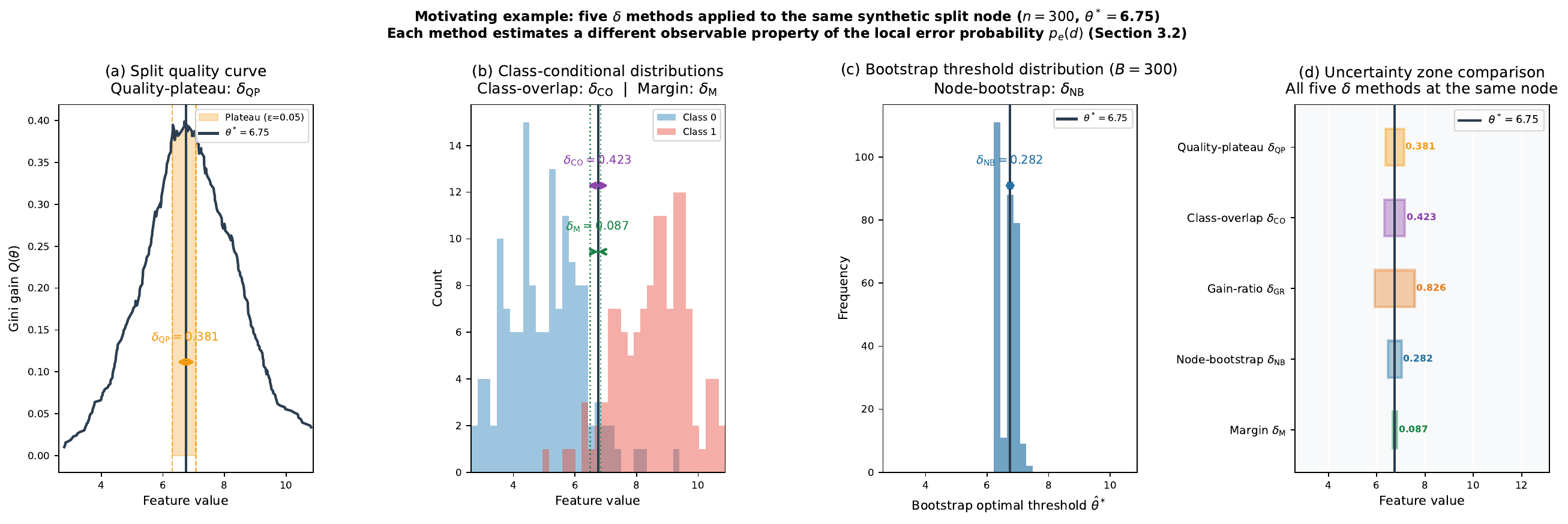}
\caption{Five $\delt$ methods applied to the same synthetic split node
($n=300$, $\thet^*=6.75$, classes $\mathcal{N}(5,1.5)$ and
$\mathcal{N}(9,1.5)$).
(a) Split quality curve with near-optimal plateau highlighted;
$\delt_{\mathrm{QP}}$ is the half-width of this plateau.
(b) Class-conditional distributions; $\delt_{\mathrm{CO}}$ captures
distributional overlap and $\delt_{\mathrm{M}}$ the distance to the
nearest cross-class example.
(c) Bootstrap distribution of optimal thresholds across $B=300$
resamples; $\delt_{\mathrm{NB}}$ equals the standard deviation.
(d) Resulting uncertainty zones $[\thet^*-\delt,\,\thet^*+\delt]$
for all five methods, illustrating their differing sensitivity.
Under the decision-theoretic framework, each method estimates a
different aspect of the local error probability $p_e(d)$.}
\label{fig:motivating}
\end{figure}

\subsection{Five Methods for Computing \texorpdfstring{$\delt$}{delta}}
\label{sec:deltamethods}

Each of the five methods below is a tractable approximation to the
theoretically optimal $\delt^*$ of Equation~\eqref{eq:delta_optimal}.
Because $p_e(d)$ is not directly observable from finite training data,
each method estimates a different surrogate quantity that correlates
with the local classification error near $\thet^*$, as detailed in
Table~\ref{tab:estimators}.
Critically, four of the five methods (quality-plateau, gain-ratio,
class-overlap, and margin) use only statistics already computed during
CART's own split evaluation --- quality scores, class distributions, and
candidate thresholds --- incurring zero additional data passes and
negligible computational overhead over standard CART.
Node-bootstrap is the sole exception, requiring $B$ bootstrap resamples
per node; it is excluded on large datasets ($N > 20{,}000$) for this
reason.

All five methods receive as input the column of feature values
$\mathbf{c} \in \mathbb{R}^n$, the class labels $\mathbf{y} \in
\{0,\ldots,K-1\}^n$, the sample weights $\mathbf{w} \in \mathbb{R}^n$,
the optimal threshold $\thet^*$, and the quality scores
$\{Q(\thet_j)\}$ evaluated at all candidate thresholds $\{\thet_j\}$.
These are all quantities already computed during standard CART split
finding.
Each method returns a non-negative scalar $\delt$, clamped to at most
$25\%$ of the feature range to prevent degenerate cases.

\paragraph{Quality Plateau ($\delt_\text{QP}$).}
The quality curve $Q(\thet)$ typically has a plateau of near-optimal
thresholds around $\thet^*$.
The width of this plateau reflects how precisely the data constrains the
optimal split location:
\begin{equation}
\delt_\text{QP} = \frac{\thet_\text{hi} - \thet_\text{lo}}{2},
\quad \text{where } Q(\thet) \ge (1 - \varepsilon) Q(\thet^*)
\text{ for } \thet \in [\thet_\text{lo}, \thet_\text{hi}].
\label{eq:qp}
\end{equation}
The tolerance $\varepsilon \in [0,1]$ controls the plateau width; we use
$\varepsilon = 0.05$.
This method adds only $O(T)$ work per node by scanning the quality curve
already evaluated, where $T$ is the number of candidate thresholds.

\paragraph{Class Overlap ($\delt_\text{CO}$).}
The width of the region where class-conditional feature distributions
overlap directly measures where neither class owns the feature space:
\begin{equation}
\delt_\text{CO} = \frac{1}{2} \max_{c \ne c'} \left[
  \min\!\left(P_{1-q}^{(c)},\, P_{1-q}^{(c')}\right) -
  \max\!\left(P_q^{(c)},\, P_q^{(c')}\right)
\right]^+\!,
\label{eq:co}
\end{equation}
where $P_q^{(c)}$ is the $q$-th percentile of feature values for class
$c$ at this node, and $[\cdot]^+ = \max(0, \cdot)$.
We use $q = 0.10$.

\paragraph{Gain Ratio ($\delt_\text{GR}$).}
The gain ratio $\text{GR} = \text{IG}(\thet^*) / H_\text{split}(\thet^*)$
measures how informative the split is relative to its balance.
High gain ratio indicates a clear, confident split; low gain ratio
indicates ambiguity:
\begin{equation}
\delt_\text{GR} = \frac{\alpha \cdot \text{range}(\mathbf{c})}{1 + \text{GR}},
\label{eq:gr}
\end{equation}
where $\alpha$ is a global scale parameter (default $\alpha = 0.10$) and
$H_\text{split}(\thet^*)$ is the binary entropy of the split proportions.
This method adds zero computation beyond what CART already performs.

\paragraph{Node Bootstrap ($\delt_\text{NB}$).}
The most statistically principled method directly measures threshold
instability by resampling the node's training examples:
\begin{equation}
\delt_\text{NB} = \mathrm{std}\!\left(\thet^*_1, \ldots, \thet^*_B\right),
\label{eq:nb}
\end{equation}
where $\thet^*_b$ is the optimal threshold on the $b$-th bootstrap
resample of the $n$ examples at this node.
We scale $B$ with dataset size: $B = 20$ for $n < 2000$, $B = 15$ for
$n < 10000$, and $B = 10$ otherwise.
This method is excluded on datasets with $N > 20{,}000$ due to
computational cost.

\paragraph{Margin ($\delt_\text{M}$).}
Inspired by support vector machines~\cite{cortes1995svm}, we define
$\delt$ as the distance from $\thet^*$ to the nearest cross-class
training example (the ``support vector'' at this node):
\begin{equation}
\delt_\text{M} = \min\!\left(
  \min_{i: x_{if} > \thet^*, y_i = c_L} (x_{if} - \thet^*),\;
  \min_{i: x_{if} \le \thet^*, y_i = c_R} (\thet^* - x_{if})
\right),
\label{eq:margin}
\end{equation}
where $c_L$ and $c_R$ are the dominant classes on the left and right
sides of $\thet^*$ respectively.
When the split is perfectly clean (no cross-class examples), we fall
back to the physical gap between the two adjacent data points straddling
$\thet^*$.
This method requires zero hyperparameters and $O(n)$ work per node.

\subsection{Routing Architectures}
\label{sec:routing}

\paragraph{Probabilistic routing (binary tree structure).}
The tree is trained identically to standard CART, with all examples
routed left ($\le \thet^*$) or right ($> \thet^*$) during fitting.
The value $\delt$ is computed and stored at each node for use during
prediction only.
At prediction time, an instance $\mathbf{x}$ arriving at a node with
feature $f$, threshold $\thet$, and uncertainty zone $\delt$ is routed
as follows:
\begin{itemize}
\item $x_f \le \thet - \delt$: route left (decisive)
\item $x_f > \thet + \delt$: route right (decisive)
\item otherwise: propagate to \emph{both} children with weights
  $w_L = (\thet + \delt - x_f)/(2\delt)$ and
  $w_R = 1 - w_L$, combining probability outputs as
  $\hat{p} = w_L \hat{p}_L + w_R \hat{p}_R$.
\end{itemize}
The instance is flagged as boundary-uncertain if it enters the weighted
combination at any node.

\paragraph{Hard middle branch (trinary tree structure).}
A third child is added to each node.
During \emph{training}, examples with $\thet - \delt < x_f \le \thet
+ \delt$ are routed to this middle child and used to train a separate
subtree; left and right subtrees are trained on examples outside the
uncertainty zone.
During \emph{prediction}, the instance is routed to whichever of the
three branches contains its feature value and is flagged as
boundary-uncertain if it enters the middle branch.

\paragraph{Equivalence of probabilistic and deferred routing.}
We also tested \emph{deferred routing}, in which the weighted
combination is resolved using hard binary routing within child subtrees.
Across all 72 CC18 datasets and all five $\delt$ methods, probabilistic
and deferred routing produced numerically identical results (to four
decimal places).
At depth $\le 4$, the probability of a single instance traversing two or
more uncertainty zones on the same path is negligible in practice.
We report only probabilistic routing in all experiments.

\subsection{Prediction and Boundary-Uncertain Flagging}

A standard class prediction is always produced by taking the argmax of
the (possibly blended) leaf probability vector.
The ternary verdict is \emph{not abstention}: it is a qualifier on the
prediction indicating whether the path to the leaf was decisive or
passed through a weighted combination.
Formally:
\begin{equation}
\text{verdict}(\mathbf{x}) = \begin{cases}
  \TRUE  & \text{if no uncertainty zone was traversed} \\
  \UND   & \text{if at least one uncertainty zone was traversed}
\end{cases}
\label{eq:verdict}
\end{equation}

We evaluate performance on \emph{decided} instances (verdict $= \TRUE$)
separately from all instances, using two primary metrics:
\textbf{Decided Accuracy} (accuracy restricted to decided instances) and
\textbf{Boundary-Uncertain Rate} (fraction of instances receiving verdict
$\UND$).

\section{Theoretical Analysis}
\label{sec:theory}

We establish four formal properties of ternary decision trees:
an accuracy decomposition, a sufficiency condition for decided
accuracy improvement, an exact characterisation of the efficiency
metric, and asymptotic consistency of the margin method.

\subsection{Accuracy Decomposition}
\label{sec:theory_decomp}

\begin{proposition}[Accuracy Decomposition]
\label{prop:decomp}
Let $u \in [0,1)$ be the boundary-uncertain rate,
$\decacc$ the decided accuracy,
$\mathrm{Acc}_u$ the accuracy on boundary-uncertain instances, and
$\mathrm{Acc}$ the overall accuracy on all instances.
Then
\begin{equation}
\mathrm{Acc} \;=\; (1-u)\,\decacc + u\,\mathrm{Acc}_u.
\label{eq:decomp}
\end{equation}
\end{proposition}

\begin{proof}
Partition the test set of $N$ instances into the decided subset
$\mathcal{D}$ ($|\mathcal{D}| = (1-u)N$) and the boundary-uncertain
subset $\mathcal{U}$ ($|\mathcal{U}| = uN$).
By definition of accuracy:
\begin{align*}
\mathrm{Acc}
  &= \frac{|\{i \in \mathcal{D}: \hat{y}_i = y_i\}|
         + |\{i \in \mathcal{U}: \hat{y}_i = y_i\}|}{N}\\
  &= \frac{(1-u)N\,\decacc + uN\,\mathrm{Acc}_u}{N}
  = (1-u)\,\decacc + u\,\mathrm{Acc}_u. \qed
\end{align*}
\end{proof}

\subsection{Sufficiency Condition for Decided Accuracy Improvement}
\label{sec:theory_sufficiency}

\begin{proposition}[Sufficiency Condition]
\label{prop:sufficiency}
Let $\mathrm{Acc}_{\mathrm{CART}}$ denote the overall accuracy of the
CART baseline.
Assume overall accuracy is preserved under ternary routing:
$\mathrm{Acc} \approx \mathrm{Acc}_{\mathrm{CART}}$.
Then the ternary decision tree achieves higher decided accuracy than
CART if and only if the accuracy on boundary-uncertain instances is
strictly below the CART baseline accuracy:
\begin{equation}
\decacc > \mathrm{Acc}_{\mathrm{CART}}
\;\iff\;
\mathrm{Acc}_u < \mathrm{Acc}_{\mathrm{CART}}.
\label{eq:sufficiency}
\end{equation}
\end{proposition}

\begin{proof}
Applying Proposition~\ref{prop:decomp} under the assumption
$\mathrm{Acc} = \mathrm{Acc}_{\mathrm{CART}}$:
\begin{equation*}
\decacc = \frac{\mathrm{Acc}_{\mathrm{CART}} - u\,\mathrm{Acc}_u}{1-u}.
\end{equation*}
Then $\decacc > \mathrm{Acc}_{\mathrm{CART}}$ iff
\begin{align*}
\mathrm{Acc}_{\mathrm{CART}} - u\,\mathrm{Acc}_u
  &> \mathrm{Acc}_{\mathrm{CART}}\,(1-u)\\
u\,\mathrm{Acc}_{\mathrm{CART}} &> u\,\mathrm{Acc}_u\\
\mathrm{Acc}_u &< \mathrm{Acc}_{\mathrm{CART}},
\end{align*}
where the last step uses $u > 0$. \qed
\end{proof}

\begin{remark}
Proposition~\ref{prop:sufficiency} has a direct practical
interpretation: decided accuracy improvement over CART requires the
boundary-uncertain instances to be harder than the CART baseline
average.
A delta method that flags genuinely hard instances (high local $p_e$,
near the true decision boundary) satisfies this condition.
A method that flags easy instances (low $p_e$) does not.
This provides theoretical grounding for the empirical finding that
class-overlap and gain-ratio achieve nominally high decided accuracy
by routing the majority of instances through the boundary-uncertain
zone rather than by precisely identifying hard cases
(Section~\ref{sec:cc18}).
\end{remark}

\subsection{Efficiency Characterisation}
\label{sec:theory_efficiency}

We define the \emph{efficiency ratio}
$\eta = (\decacc - \mathrm{Acc}_{\mathrm{CART}}) / u$
as the decided accuracy gain per unit of boundary-uncertain rate.

\begin{proposition}[Efficiency Characterisation]
\label{prop:efficiency}
Under the assumption of Proposition~\ref{prop:sufficiency},
\begin{equation}
\eta \;=\; \decacc - \mathrm{Acc}_u.
\label{eq:efficiency_char}
\end{equation}
The efficiency ratio equals the accuracy gap between decided and
boundary-uncertain predictions.
\end{proposition}

\begin{proof}
From Proposition~\ref{prop:sufficiency}:
$\decacc - \mathrm{Acc}_{\mathrm{CART}}
= u(\mathrm{Acc}_{\mathrm{CART}} - \mathrm{Acc}_u)\,/\,(1-u)$,
so $\eta = (\mathrm{Acc}_{\mathrm{CART}} - \mathrm{Acc}_u)\,/\,(1-u)$.
Independently, $\decacc - \mathrm{Acc}_u
= (\mathrm{Acc}_{\mathrm{CART}} - \mathrm{Acc}_u)\,/\,(1-u)$.
Therefore $\eta = \decacc - \mathrm{Acc}_u$. \qed
\end{proof}

\begin{remark}
Equation~\eqref{eq:efficiency_char} gives the efficiency ratio a
concrete interpretation: it measures how much more accurately
the tree classifies decided instances relative to boundary-uncertain
ones.
High $\eta$ indicates that the delta method successfully separates
easy and hard instances; $\eta \le 0$ indicates the flagged
instances are no harder than the decided set.
Since the reported tables include $\mathrm{Acc}$, $\decacc$, and $u$,
the boundary-uncertain accuracy $\mathrm{Acc}_u$ is recoverable via
Proposition~\ref{prop:decomp}:
$\mathrm{Acc}_u = (\mathrm{Acc} - (1-u)\,\decacc)\,/\,u$,
allowing the efficiency characterisation to be verified directly
from reported values.
\end{remark}

\subsection{Asymptotic Consistency of the Margin Method}
\label{sec:theory_consistency}

\begin{proposition}[Asymptotic Consistency]
\label{prop:consistency}
Let $n_v$ denote the number of training examples at node $v$.
As $n_v \to \infty$, the margin-based $\delt_M \to 0$ in probability,
provided the class-conditional feature densities are bounded and
non-zero in a neighbourhood of $\thet^*$.
Consequently, the margin-based ternary tree recovers standard CART
behaviour asymptotically and inherits its consistency guarantees.
\end{proposition}

\begin{proof}[Proof sketch]
The margin $\delt_M$ equals the distance from $\thet^*$ to the
nearest cross-class training example.
By hypothesis each class has positive density near $\thet^*$,
so both classes have examples on both sides of $\thet^*$ with
probability approaching 1 as $n_v \to \infty$.
The minimum cross-class distance converges to zero in probability
at rate $O(n_v^{-1})$.
As $\delt_M \to 0$ the boundary-uncertain rate $u \to 0$ and the
ternary tree reduces to standard CART, whose consistency under
growing-depth conditions (depth $d_n \to \infty$, $d_n = o(\log n)$)
is established in~\cite{devroye1996probabilistic}. \qed
\end{proof}

\begin{remark}
Proposition~\ref{prop:consistency} does not hold for class-overlap
or gain-ratio in general, since those methods may produce $\delt > 0$
even when classes are perfectly separable at the node.
Node-bootstrap shares the consistency property under regularity
conditions (bootstrap threshold variance converges to zero as
$n_v \to \infty$ when the optimal threshold is unique).
The practical experiments in Section~\ref{sec:experiments} use
fixed-depth trees ($\texttt{max\_depth}=4$), which do not satisfy
the growing-depth condition required for Bayes consistency of the
complete tree; Proposition~\ref{prop:consistency} characterises the
structural relationship between the ternary and CART trees at each
individual node as $n_v \to \infty$, independently of the global
depth-growth schedule.
\end{remark}

\section{Experiments}
\label{sec:experiments}

\subsection{Datasets and Collections}

We evaluate across three benchmark collections covering 78 datasets in total.

\paragraph{OpenML-CC18.}
The curated benchmark suite of 72 classification
tasks~\cite{bischl2021cc18} covers diverse dataset sizes
($N \in [10, 130{,}000]$), dimensionalities ($d \in [3, 3{,}072]$),
and class structures ($K \in [2, 10]$).
One dataset (task 167124, raw image pixels, $d = 3{,}072$) was excluded
as axis-aligned tree splits are not meaningful on flattened pixel arrays,
leaving 71 evaluated datasets.

\paragraph{Breiman synthetic benchmarks.}
Three datasets with analytically-known Bayes errors from OpenML:
\emph{waveform-5000} ($N = 5{,}000$, $d = 21$, $K = 3$,
Bayes error $\approx 14\%$)~\cite{breiman1984classification},
\emph{twonorm} ($N = 7{,}400$, $d = 20$, $K = 2$,
Bayes error $\approx 2.3\%$)~\cite{breiman1996bagging},
and \emph{ringnorm} ($N = 7{,}400$, $d = 20$, $K = 2$,
Bayes error $\approx 1.7\%$)~\cite{breiman1996bagging}.

\paragraph{Medical and financial datasets.}
Four high-stakes binary classification tasks from OpenML:
Pima diabetes ($N = 768$, $d = 8$),
German credit ($N = 1{,}000$, $d = 20$),
Cleveland heart disease ($N = 303$, $d = 13$),
and Mammography ($N = 11{,}183$, $d = 6$, severely imbalanced at $2.3\%$
positive).

\subsection{Experimental Protocol}

All experiments use 5-fold stratified cross-validation.
Features are standardised within each fold (StandardScaler fitted on
training data only).
All ternary trees use $\texttt{max\_depth} = 4$.
Node-bootstrap is excluded on datasets with $N > 20{,}000$ due to
computational cost; this affects a subset of the CC18 datasets and is
noted in Table~\ref{tab:cc18}.
We compare against \textbf{sklearn CART} (DecisionTreeClassifier,
$\texttt{max\_depth} = 4$) as the baseline.

We report \textbf{Decided Accuracy} (primary metric),
\textbf{Boundary-Uncertain Rate} (\%), \textbf{Overall Accuracy} (all
instances), and \textbf{Decided F1} (macro-averaged F1 on decided
instances).
Statistical significance of decided accuracy improvements over the
baseline is assessed using the one-sided Wilcoxon signed-rank
test~\cite{demsar2006statistical} across per-dataset mean decided
accuracies.
Win/Tie/Loss (W/T/L) counts use a practical significance threshold of
$0.5$ percentage points.

\section{Results}
\label{sec:results}

\subsection{OpenML-CC18: Broad Validity}
\label{sec:cc18}

Table~\ref{tab:cc18} presents results across 71 of the 72 OpenML-CC18
datasets (one image dataset with $d=3{,}072$ excluded; see
Section~\ref{sec:experiments}).

\begin{table}[t]
\caption{Decided accuracy across OpenML-CC18 (71 of 72 datasets evaluated,
5-fold CV; one image dataset with $d=3{,}072$ raw pixel features excluded).
W/T/L = wins/ties/losses vs CART baseline on decided accuracy.
Node-bootstrap excluded where $N > 20{,}000$ due to computational cost.
All probabilistic-routing methods are significant at $p \le 0.001$
(Wilcoxon, one-sided).}
\label{tab:cc18}
\centering
\small
\begin{tabular}{llccccr}
\toprule
Method & Routing & Dec.Acc & $\pm$std & Undec\% & Acc.All & W/T/L \\
\midrule
Class-Overlap  & prob. & 0.7620 & 0.207 & 71.3 & 0.7177 & 48/7/17 \\
Gain-Ratio     & prob. & 0.7583 & 0.209 & 58.2 & 0.7186 & 45/14/13 \\
Node-Bootstrap & h.m.  & 0.7425 & 0.193 & 17.6 & 0.7280 & 38/17/7 \\
Node-Bootstrap & prob. & 0.7421 & 0.193 & 31.9 & 0.7336 & 36/11/15 \\
\textbf{Margin} & \textbf{prob.} & \textbf{0.7397} & \textbf{0.193} & \textbf{16.7} & \textbf{0.7220} & \textbf{42/21/9} \\
Quality-Plateau & prob. & 0.7337 & 0.189 & 17.4 & 0.7234 & 37/29/6 \\
CART (baseline) & ref.  & 0.7217 & 0.194 &  0.0 & 0.7217 & ref. \\
Margin & h.m.  & 0.7203 & 0.202 &  9.7 & 0.7136 & 35/23/14 \\
Quality-Plateau & h.m.  & 0.7128 & 0.224 & 11.0 & 0.7050 & 35/27/10 \\
Gain-Ratio     & h.m.  & 0.7015 & 0.238 & 27.1 & 0.6838 & 33/16/23 \\
Class-Overlap  & h.m.  & 0.6785 & 0.280 & 32.9 & 0.6350 & 36/9/27 \\
\bottomrule
\end{tabular}
\end{table}

\begin{figure}[t]
\centering
\includegraphics[width=0.92\textwidth]{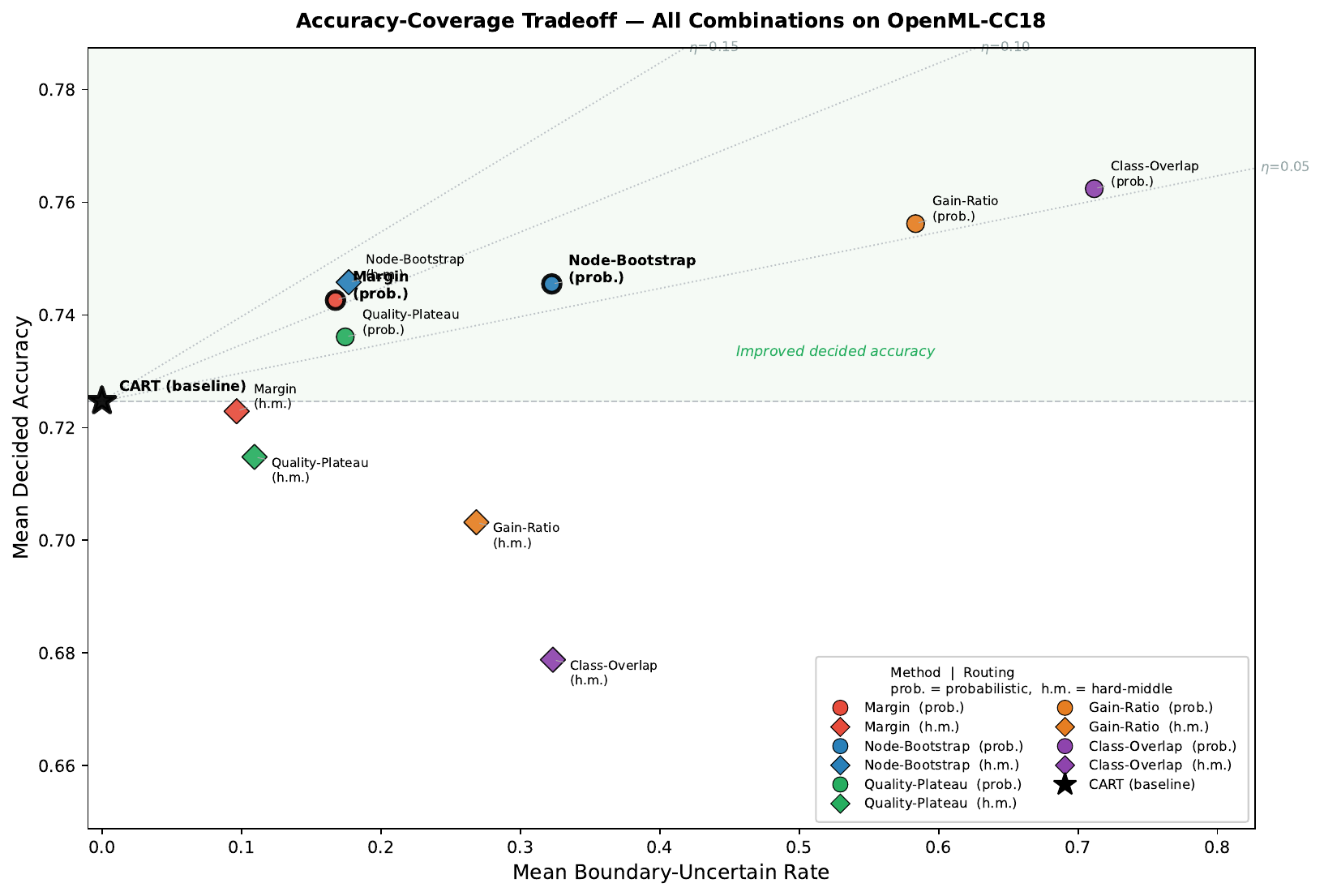}
\caption{Accuracy-coverage tradeoff across 71 of the 72 OpenML-CC18 datasets
(5-fold CV means).
Each point represents one (delta method, routing) combination at its
mean boundary-uncertain rate ($x$-axis) and mean decided accuracy
($y$-axis).
The dashed horizontal line marks the CART baseline decided accuracy.
Efficiency isolines ($\eta$ = decided accuracy gain per unit of
boundary-uncertain rate) show that methods above and to the left
of the baseline deliver improvement with lower abstention cost.
Circles denote probabilistic routing; diamonds denote hard-middle routing.
The margin method (red circle) achieves the best efficiency
($\eta = 0.104$), combining competitive decided accuracy with a
moderate boundary-uncertain rate of 16.7\%.
Class-overlap and gain-ratio (far right) achieve high decided
accuracy by routing the majority of instances through the
uncertainty zone, as reflected by their low efficiency scores.}
\label{fig:tradeoff}
\end{figure}

\paragraph{Significance.}
All five $\delt$ methods with probabilistic routing significantly
outperform the CART baseline on decided accuracy ($p \le 0.001$,
Wilcoxon signed-rank, one-sided), confirming that ternary decision trees
with locally-adaptive uncertainty zones provide a statistically robust
improvement across diverse datasets.

\paragraph{Class-Overlap and Gain-Ratio.}
Although class-overlap and gain-ratio achieve nominally higher decided
accuracy (0.762 and 0.758 respectively) than margin (0.740), this is
attributable to selective routing: by flagging 71\% and 58\% of
instances as boundary-uncertain respectively, these methods classify only
the easiest instances decisively.
We quantify this with the efficiency metric
$\eta = (\text{Dec-Acc} - \text{CART-Acc}) / \text{Undec\%}$:
class-overlap achieves $\eta = 0.068$ and gain-ratio $\eta = 0.072$,
compared to $\eta = 0.104$ for margin.
These methods are therefore not recommended for practical use in their
current parameterisation.

\paragraph{Margin is the recommended method.}
The margin method achieves the highest efficiency ($\eta = 0.104$),
flags only 16.7\% of instances as boundary-uncertain, wins on 42 of 72
datasets ($p < 0.001$), and requires zero additional hyperparameters.
This combination of performance, parsimony, and theoretical grounding
from SVM margin theory makes it the recommended $\delt$ method.

\paragraph{Hard-middle routing.}
Probabilistic routing outperforms hard-middle routing for 4 of 5
$\delt$ methods.
In hard-middle routing, boundary-uncertain examples are removed from the
left and right training sets, reducing the quality of those branches.
The resulting negative efficiency ($\eta < 0$ for all hard-middle
combinations on CC18) confirms that this structural cost exceeds any
benefit on diverse datasets.
Node-bootstrap is the exception, performing comparably under both routing
strategies.

\subsection{Breiman Benchmarks: Theoretical Analysis}
\label{sec:breiman}

Table~\ref{tab:breiman} presents results on the three Breiman synthetic
datasets.
We introduce the \textbf{Undecided/Bayes (U/B) ratio} as a diagnostic:
$\text{U/B} = \text{Undec\%} / \text{Bayes error}$.
A ratio close to 1.0 indicates the method flags approximately the
fraction of instances that are theoretically ambiguous; ratios much
greater than 1 indicate over-flagging.

\begin{table}[t]
\caption{Breiman synthetic benchmarks with analytically-known Bayes errors.
U/B = Boundary-Uncertain Rate / Bayes Error. A ratio closer to 1.0
indicates better alignment with the theoretically ambiguous region.
No method achieves calibration, motivating future work on $\delta$
scaling.}
\label{tab:breiman}
\centering
\small
\begin{tabular}{lcccccc}
\toprule
& \multicolumn{2}{c}{waveform (BE=14\%)}
& \multicolumn{2}{c}{twonorm (BE=2.3\%)}
& \multicolumn{2}{c}{ringnorm (BE=1.7\%)} \\
\cmidrule(lr){2-3}\cmidrule(lr){4-5}\cmidrule(lr){6-7}
Method & Dec.Acc & U/B & Dec.Acc & U/B & Dec.Acc & U/B \\
\midrule
Node-Bootstrap & 0.7845 & 2.57 & 0.8245 & 15.64 & 0.7688 & 11.77 \\
\textbf{Margin} & \textbf{0.7377} & \textbf{0.09} & \textbf{0.7774} & \textbf{0.42} & \textbf{0.7648} & \textbf{0.39} \\
Quality-Plateau & 0.7711 & 2.54 & 0.8079 & 13.49 & 0.7879 & 10.68 \\
CART (baseline) & 0.7350 & ref. & 0.7759 & ref. & 0.7647 & ref. \\
\bottomrule
\end{tabular}
\end{table}

\paragraph{Margin is self-calibrating on clean data.}
The margin method produces U/B ratios of 0.09, 0.42, and 0.39 on the
three datasets, all well below 1.0.
On geometrically structured synthetic data, the SVM-inspired gap between
class support vectors at each node tends to be very small because both
classes densely populate the feature space near the boundary.
Margin correctly identifies almost no instances as threshold-uncertain,
producing predictions nearly identical to CART.
This is desirable behaviour: a method that aggressively flags uncertainty
on geometrically clean data would produce false signals.

\paragraph{Node-bootstrap and quality-plateau are proportionally better calibrated.}
Node-bootstrap and quality-plateau achieve U/B values of 2.4 to 2.6 on
waveform (the dataset with the largest Bayes error at 14\%) compared to
values of 10 to 16 on twonorm and ringnorm (Bayes errors of 2.3\% and
1.7\%).
While all ratios exceed 1.0, the relative ordering is correct: these
methods flag proportionally more instances on harder datasets.

\paragraph{No method achieves calibration.}
No $\delt$ method achieves U/B $\approx 1.0$, confirming that local
threshold uncertainty as computed from split statistics does not
directly correspond to the Bayes-uncertain fraction of the dataset.
This is expected: the Bayes error characterises the joint distribution
in the full feature space, while $\delt$ characterises uncertainty about
individual univariate split thresholds.
We identify calibration of $\delt$ to target the Bayes error boundary
as a productive direction for future work, for example via the
Cover-Hart $k$-NN Bayes error bound~\cite{coverhart1967}.

\subsection{Medical and Financial Datasets: Domain Value}
\label{sec:medical}

Table~\ref{tab:medical} presents results on four high-stakes datasets
for the three recommended methods and the baseline.

\begin{table}[t]
\caption{Medical and financial datasets.
Dec.Acc = accuracy on decisively classified instances.
Undec\% = fraction of instances classified as boundary-uncertain
(these instances still receive a prediction, formed by weighted subtree
blending).
F1-Dec = macro-averaged F1 on decisive instances.}
\label{tab:medical}
\centering
\small
\begin{tabular}{llcccc}
\toprule
Dataset & Method & Dec.Acc & Undec\% & Acc.All & F1-Dec \\
\midrule
\multirow{4}{*}{credit-g}
 & Node-Bootstrap & 0.7677 & 71.0 & 0.7180 & 0.614 \\
 & \textbf{Margin} & \textbf{0.6808} & \textbf{48.1} & \textbf{0.7130} & \textbf{0.584} \\
 & Quality-Plateau & 0.7126 & 7.7 & 0.7130 & 0.607 \\
 & CART (baseline) & 0.7130 & 0.0 & 0.7130 & 0.605 \\
\midrule
\multirow{4}{*}{diabetes}
 & Node-Bootstrap & 0.8363 & 63.4 & 0.7499 & 0.814 \\
 & \textbf{Margin} & \textbf{0.7418} & \textbf{9.5} & \textbf{0.7343} & \textbf{0.715} \\
 & Quality-Plateau & 0.7772 & 29.4 & 0.7408 & 0.751 \\
 & CART (baseline) & 0.7382 & 0.0 & 0.7382 & 0.714 \\
\midrule
\multirow{4}{*}{heart-c}
 & Node-Bootstrap & 0.7102 & 60.7 & 0.7407 & 0.640 \\
 & \textbf{Margin} & \textbf{0.7744} & \textbf{58.9} & \textbf{0.7296} & \textbf{0.508} \\
 & Quality-Plateau & 0.7296 & 0.0 & 0.7296 & 0.723 \\
 & CART (baseline) & 0.7370 & 0.0 & 0.7370 & 0.730 \\
\midrule
\multirow{4}{*}{mammography}
 & Node-Bootstrap & 0.9915 & 10.8 & 0.9848 & 0.764 \\
 & \textbf{Margin} & \textbf{0.9867} & \textbf{0.7} & \textbf{0.9845} & \textbf{0.769} \\
 & Quality-Plateau & 0.9880 & 5.5 & 0.9847 & 0.789 \\
 & CART (baseline) & 0.9844 & 0.0 & 0.9844 & 0.775 \\
\bottomrule
\end{tabular}
\end{table}

\paragraph{Mammography: strongest practical result.}
On the severely imbalanced mammography dataset (2.3\% positive rate),
node-bootstrap flags 10.8\% of screening cases as boundary-uncertain
while achieving $+0.71\%$ decided accuracy over CART.
In a clinical screening context, these flagged cases correspond to
radiological findings near the classification boundary, precisely where
a second review or additional diagnostic testing is most beneficial.
Quality-plateau achieves the best F1-Dec (0.789), indicating better
handling of the class imbalance on the decisive subset.

\paragraph{Diabetes: large gains at the cost of high boundary-uncertain rates.}
Node-bootstrap achieves $+9.8\%$ decided accuracy over the 0.738
baseline but flags 63.4\% of patients as boundary-uncertain.
Margin achieves a more practical operating point: $+0.4\%$ decided
accuracy with only 9.5\% boundary-uncertain flagging.
For clinical screening where the cost of flagging is high,
margin's conservative operating point is preferable.

\paragraph{Heart disease: small dataset limitation.}
On the Cleveland heart disease dataset ($N = 303$), node-bootstrap falls
below the baseline decided accuracy (0.710 vs.\ 0.737).
With only approximately 240 training examples per fold, bootstrap-based
threshold variance estimates are unreliable.
We recommend a minimum sample threshold of $N \ge 500$ for the
node-bootstrap method.
Quality-plateau's decided accuracy closely matches the baseline (0.730
vs.\ 0.737) without requiring a minimum sample size.
Margin achieves the highest decided accuracy (0.774) on this dataset,
though at a high boundary-uncertain rate (58.9\%) that limits practical
utility.

\paragraph{German credit: margin limitation.}
Margin is the only method and the only dataset in our experiments where
decided accuracy falls below the CART baseline (0.681 vs.\ 0.713).
German credit's 20 economic features and complex feature interactions
appear to produce split margins that do not well-characterise threshold
uncertainty.
This limitation should be considered in applications involving
small-to-medium datasets with intricate economic feature structures.

\section{Discussion}
\label{sec:discussion}

\paragraph{Practical recommendations.}
Based on our empirical analysis, we offer the following recommendations.
\textbf{Margin (prob.)} is the default choice, offering the best
efficiency, zero hyperparameters, theoretical grounding in SVM margin
theory, and self-calibrating behaviour on clean data; it is suitable for
datasets with $N \ge 200$.
\textbf{Node-bootstrap (prob.)} is recommended when maximum decided
accuracy is required and $N \ge 500$ per training fold.
Node-bootstrap adds $O(B \cdot n \log n)$ work per node where $B$
is the number of bootstrap replicates, making it substantially more
expensive than the other four delta methods which add at most $O(n)$
overhead to the standard CART split evaluation.
All other methods are computationally lightweight: quality-plateau
and gain-ratio reuse the split criterion scores already computed by
CART; margin requires a single linear scan over the node's training
examples; class-overlap requires percentile computation over class
subsets.
We exclude node-bootstrap on datasets with $N > 20{,}000$ due to
this cost; a compiled implementation would reduce this overhead
significantly.
\textbf{Quality-plateau (prob.)} is the safest choice: the most
consistent improvement across all datasets tested, never
catastrophically underperforming, and suitable for any dataset size.
\textbf{Probabilistic routing} is preferred in all cases, as hard-middle
routing is consistently inferior for 4 of 5 methods on the CC18
benchmark.

\paragraph{Theoretical grounding.}
The empirical findings are supported by formal guarantees established
in Section~\ref{sec:theory}.
Proposition~\ref{prop:sufficiency} provides the theoretical basis for
decided accuracy improvement: a ternary tree improves over CART if and
only if the boundary-uncertain instances are harder than average ---
precisely the condition satisfied by methods that flag instances with
high local error probability $p_e$ near the decision boundary.
Proposition~\ref{prop:efficiency} reveals that the efficiency
ratio $\eta$ exactly equals $\decacc - \mathrm{Acc}_u$, the accuracy
gap between decided and boundary-uncertain predictions; the empirical
efficiency scores in Section~\ref{sec:cc18} are therefore not merely
descriptive but have a precise theoretical interpretation.
The finding that class-overlap and gain-ratio achieve high decided
accuracy through excessive flagging is a formal consequence of
Proposition~\ref{prop:sufficiency}: those methods flag instances with
low $p_e$ rather than genuinely hard ones, so their high decided
accuracy reflects selective routing rather than improved classification.
Proposition~\ref{prop:consistency} guarantees that the margin-based
ternary tree recovers standard CART asymptotically, confirming that
the ternary extension does not sacrifice the consistency properties
of its predecessor.

\paragraph{The boundary-uncertain verdict in practice.}
It is important to distinguish boundary-uncertain verdicts from
abstention.
Every instance receives a class prediction.
The boundary-uncertain flag is a qualifier indicating that the prediction
was formed by weighted blending of subtree outputs rather than
deterministic routing to a single leaf.
Applications that can act on this signal will benefit most from the
ternary architecture: scheduling additional review for flagged medical
cases, routing flagged credit applications to manual underwriters, or
returning a confidence interval rather than a point prediction are
natural use cases.

\paragraph{Calibration of the boundary-uncertain rate.}
The Undecided/Bayes ratio analysis reveals that none of the five
$\delt$ methods produces a boundary-uncertain rate that aligns with the
theoretical Bayes error.
Methods that produce high decided accuracy on the Breiman benchmarks do
so by flagging substantially more instances than are theoretically
ambiguous (U/B between 2 and 16 for node-bootstrap and quality-plateau).
Deriving a scaling of $\delt$ that targets the Bayes error boundary,
for example via the Cover-Hart $k$-NN Bayes error estimate, is a
natural extension that we leave for future work.

\section{Conclusion}
\label{sec:conclusion}

We introduced ternary decision trees with locally-adaptive uncertainty
zones, providing a principled and theoretically grounded framework for
augmenting CART with per-node threshold uncertainty.
A decision-theoretic analysis characterises the optimal zone half-width
$\delt^*$ as the solution to a node-local cost-minimisation problem,
and five tractable estimation methods are proposed and evaluated as
approximations to this optimum, each targeting a different estimable
property of the local error probability $p_e(d)$.
Two routing architectures were proposed and evaluated, and four formal
theoretical properties were established: an accuracy decomposition, a
sufficiency condition for decided accuracy improvement, an exact
characterisation of the efficiency metric, and asymptotic consistency
of the margin method.

Across 71 of the 72 OpenML-CC18 datasets, all five $\delt$ methods with
probabilistic routing significantly outperform CART on decided accuracy
($p \le 0.001$, Wilcoxon signed-rank).
The margin method, which requires zero additional hyperparameters,
achieves the best efficiency ($\eta = 0.104$) and is recommended as
the default.
The formal efficiency characterisation (Proposition~\ref{prop:efficiency})
shows that methods with low efficiency flag easy rather than hard
instances, providing a theoretical basis for the empirical finding that
class-overlap and gain-ratio improve decided accuracy through selective
routing rather than improved classification.
Analysis on Breiman synthetic benchmarks with known Bayes errors reveals
that no method precisely calibrates its boundary-uncertain flagging rate
to the theoretical ambiguity region, establishing $\delt$ calibration as
a concrete open problem.

This work establishes the ternary tree as a theoretically grounded
primitive for uncertainty-aware classification.
Natural extensions include ensemble methods for ternary trees (analogous
to random forests), $\delt$ calibration targeting the Bayes error, and
ternary rule extraction for interpretable classification with
boundary-uncertain verdicts.

\section*{AI Tools Disclosure}
Large language model tools were used to assist with manuscript 
preparation, including grammar refinement and proofreading. 
All scientific content, experimental design, analysis, 
and conclusions are the sole work of the authors.
\bibliographystyle{splncs04}
\bibliography{references}

\begin{thebibliography}{10}
\providecommand{\url}[1]{\texttt{#1}}
\providecommand{\urlprefix}{URL }
\providecommand{\doi}[1]{https://doi.org/#1}

\bibitem{avron2008roughkleene}
Avron, A., Konikowska, B.: Rough sets and 3-valued logics. Studia Logica
  \textbf{90}(1),  69--92 (2008)

\bibitem{bartlett2008abstaining}
Bartlett, P.L., Wegkamp, M.H.: Classification with a reject option using a
  hinge loss. Journal of Machine Learning Research  \textbf{9},  1823--1840
  (2008)

\bibitem{bischl2021cc18}
Bischl, B., Casalicchio, G., Feurer, M., Hutter, F., Lang, M., Mantovani, R.G.,
  van Rijn, J.N., Vanschoren, J.: {OpenML} benchmarking suites. arXiv preprint
  arXiv:1708.03731  (2021)

\bibitem{breiman1996bagging}
Breiman, L.: Bagging predictors. Machine Learning  \textbf{24}(2),  123--140
  (1996)

\bibitem{breiman1984classification}
Breiman, L., Friedman, J., Olshen, R., Stone, C.: Classification and Regression
  Trees. Wadsworth (1984)

\bibitem{chow1970optimum}
Chow, C.K.: On optimum recognition error and reject tradeoff. IEEE Transactions
  on Information Theory  \textbf{16}(1),  41--46 (1970).
  \doi{10.1109/TIT.1970.1054406}

\bibitem{cortes1995svm}
Cortes, C., Vapnik, V.: Support-vector networks. Machine Learning
  \textbf{20}(3),  273--297 (1995)

\bibitem{coverhart1967}
Cover, T.M., Hart, P.: Nearest neighbor pattern classification. IEEE
  Transactions on Information Theory  \textbf{13}(1),  21--27 (1967).
  \doi{10.1109/TIT.1967.1053964}

\bibitem{demsar2006statistical}
Dem{\v{s}}ar, J.: Statistical comparisons of classifiers over multiple data
  sets. Journal of Machine Learning Research  \textbf{7},  1--30 (2006)

\bibitem{devroye1996probabilistic}
Devroye, L., Gy{\"o}rfi, L., Lugosi, G.: A Probabilistic Theory of Pattern
  Recognition. Springer, New York (1996). \doi{10.1007/978-1-4612-0711-5}

\bibitem{frosst2017softdt}
Frosst, N., Hinton, G.: Distilling a neural network into a soft decision tree.
  arXiv preprint arXiv:1711.09784  (2017)

\bibitem{greco2001drsa}
Greco, S., Matarazzo, B., S{\l}owi{\'n}ski, R.: Rough sets theory for
  multicriteria decision analysis. European Journal of Operational Research
  \textbf{129}(1),  1--47 (2001)

\bibitem{irsoy2012softdt}
Irsoy, O., Yildiz, O.T., Alpaydin, E.: Soft decision trees. In: Proceedings of
  the 21st International Conference on Pattern Recognition. pp. 1121--1124
  (2012)

\bibitem{kent2022indecision}
Kent, J.S., M{\'e}nager, D.H.: Indecision trees: Learning argument-based
  reasoning under quantified uncertainty. In: Synthetic Data for Artificial
  Intelligence and Machine Learning: Tools, Techniques, and Applications.
  Proceedings of SPIE, vol. 12529, pp. 296--307. SPIE (2023).
  \doi{10.1117/12.2663380}

\bibitem{olaru2003fuzzy}
Olaru, C., Wehenkel, L.: A complete fuzzy decision tree technique. Fuzzy Sets
  and Systems  \textbf{138}(2),  221--254 (2003)

\bibitem{pawlak1982roughsets}
Pawlak, Z.: Rough sets. International Journal of Computer \& Information
  Sciences  \textbf{11}(5),  341--356 (1982)

\bibitem{pedregosa2011scikit}
Pedregosa, F., et~al.: Scikit-learn: Machine learning in Python, vol.~12 (2011)

\bibitem{quinlan1993c4}
Quinlan, J.R.: {C4.5}: Programs for Machine Learning. Morgan Kaufmann (1993)

\bibitem{vanschoren2014openml}
Vanschoren, J., van Rijn, J.N., Bischl, B., Torgo, L.: {OpenML}: Networked
  science in machine learning. ACM SIGKDD Explorations Newsletter
  \textbf{15}(2),  49--60 (2014)

\bibitem{yao2010threeway}
Yao, Y.: Three-way decisions with probabilistic rough sets. Information
  Sciences  \textbf{180}(3),  341--353 (2010)

\bibitem{yao2012twdsurvey}
Yao, Y.: An outline of a theory of three-way decisions. In: Rough Sets and
  Current Trends in Computing (RSCTC 2012). Lecture Notes in Computer Science,
  vol.~7413, pp. 1--17. Springer (2012). \doi{10.1007/978-3-642-32115-3_1}

\bibitem{zaffalon2002credal}
Zaffalon, M.: The naive credal classifier. Journal of Statistical Planning and
  Inference  \textbf{105}(1),  5--21 (2002)

\bibitem{zhi2022fuzzy3wd}
Zhi, H., Leung, Y., Liu, J.: Three-way classification: Ambiguity and abstention
  in machine learning. In: Rough Sets -- International Joint Conference, IJCRS
  2019. Lecture Notes in Computer Science, vol. 11499, pp. 280--294. Springer
  (2019). \doi{10.1007/978-3-030-22815-6_22}

\end{thebibliography}

\appendix

\section*{Appendix}
\addcontentsline{toc}{section}{Appendix}

\section{Per-Dataset Results on OpenML-CC18}
\label{app:cc18_perdataset}

Table~\ref{tab:app_cc18} reports decided accuracy and 
boundary-uncertain rate for each of the 71 evaluated 
OpenML-CC18 datasets under probabilistic routing, for the 
three recommended delta methods and the CART baseline.
Full results for all 11 combinations are given in 
Appendix~\ref{app:cc18_full}.

\begin{table}[h]
	\caption{Per-dataset decided accuracy (Dec.Acc) and 
		boundary-uncertain rate (Undec\%) on OpenML-CC18.
		Results are 5-fold CV means under probabilistic routing.
		node-bootstrap excluded where $N > 20{,}000$ (shown as --).
		Datasets sorted alphabetically.}
	\label{tab:app_cc18}
	\centering
	\scriptsize
	\begin{tabular}{lcccccccc}
		\toprule
		& \multicolumn{2}{c}{Margin}
		& \multicolumn{2}{c}{Node-Bootstrap}
		& \multicolumn{2}{c}{Quality-Plateau}
		& \multicolumn{2}{c}{CART} \\
		\cmidrule(lr){2-3}\cmidrule(lr){4-5}
		\cmidrule(lr){6-7}\cmidrule(lr){8-9}
		Dataset & Dec.Acc & Undec\%
		& Dec.Acc & Undec\%
		& Dec.Acc & Undec\%
		& Acc & \\
		\midrule
		Bioresponse,OpenML-CC18 & 0.7767 & 0.076 & 0.7521 & 0.474 & 0.7696 & 0.055 & 0.7614 & 0.000 \\
		Devnagari-Script,OpenML-CC18 & 0.1462 & 0.036 & -- & -- & 0.3655 & 0.990 & 0.1491 & 0.000 \\
		Fashion-MNIST,OpenML-CC18 & 0.6522 & 0.005 & -- & -- & 0.4100 & 0.696 & 0.6511 & 0.000 \\
		GesturePhaseSegmentationProcessed,OpenML-CC18 & 0.4747 & 0.006 & 0.4575 & 0.587 & 0.5136 & 0.403 & 0.4747 & 0.000 \\
		Internet-Advertisements,OpenML-CC18 & 0.9643 & 0.018 & 0.9654 & 0.047 & 0.9653 & 0.023 & 0.9622 & 0.000 \\
		MiceProtein,OpenML-CC18 & 0.7469 & 0.155 & 0.7998 & 0.344 & 0.7837 & 0.218 & 0.7157 & 0.000 \\
		PhishingWebsites,OpenML-CC18 & 0.9280 & 0.586 & 0.9203 & 0.000 & 0.9203 & 0.000 & 0.9203 & 0.000 \\
		adult,OpenML-CC18 & 0.8624 & 0.174 & -- & -- & 0.8729 & 0.159 & 0.8440 & 0.000 \\
		analcatdata\_authorship,OpenML-CC18 & 0.9346 & 0.145 & 0.9359 & 0.157 & 0.9242 & 0.118 & 0.9191 & 0.000 \\
		analcatdata\_dmft,OpenML-CC18 & 0.2152 & 0.657 & 0.1807 & 0.677 & 0.2058 & 0.071 & 0.2082 & 0.000 \\
		balance-scale & 0.8273 & 0.445 & 0.8471 & 0.485 & 0.7650 & 0.013 & 0.7632 & 0.000 \\
		bank-marketing,OpenML-CC18 & 0.9063 & 0.050 & -- & -- & 0.8099 & 0.715 & 0.8944 & 0.000 \\
		banknote-authentication,OpenML-CC18 & 0.9656 & 0.047 & 0.9834 & 0.324 & 0.9789 & 0.199 & 0.9548 & 0.000 \\
		blood-transfusion-service-center,OpenML-CC18 & 0.7788 & 0.221 & 0.7545 & 0.872 & 0.7557 & 0.187 & 0.7754 & 0.000 \\
		breast-w & 0.9592 & 0.144 & 0.9776 & 0.227 & 0.9514 & 0.051 & 0.9313 & 0.000 \\
		car,OpenML-CC18 & 0.9322 & 0.275 & 0.8507 & 0.000 & 0.8534 & 0.014 & 0.8507 & 0.000 \\
		churn,OpenML-CC18 & 0.9284 & 0.067 & 0.9358 & 0.144 & 0.9279 & 0.070 & 0.9204 & 0.000 \\
		climate-model-simulation-crashes,OpenML-CC18 & 0.9350 & 0.096 & 0.9658 & 0.441 & 0.9315 & 0.122 & 0.9241 & 0.000 \\
		cmc,OpenML-CC18 & 0.6152 & 0.546 & 0.5737 & 0.442 & 0.5582 & 0.071 & 0.5533 & 0.000 \\
		cnae-9,OpenML-CC18 & 0.7093 & 0.606 & 0.4315 & 0.000 & 0.4315 & 0.000 & 0.4315 & 0.000 \\
		connect-4,OpenML-CC18 & 0.7565 & 0.604 & -- & -- & 0.6761 & 0.000 & 0.6761 & 0.000 \\
		credit-approval,OpenML-CC18 & 0.8572 & 0.046 & 0.9185 & 0.696 & 0.8577 & 0.023 & 0.8551 & 0.000 \\
		credit-g,OpenML-CC18 & 0.6808 & 0.481 & 0.7677 & 0.710 & 0.7126 & 0.077 & 0.7130 & 0.000 \\
		cylinder-bands,OpenML-CC18 & 0.6887 & 0.304 & 0.6645 & 0.467 & 0.6764 & 0.065 & 0.6778 & 0.000 \\
		diabetes,OpenML-CC18 & 0.7418 & 0.095 & 0.8363 & 0.634 & 0.7772 & 0.294 & 0.7382 & 0.000 \\
		dna,OpenML-CC18 & 0.8945 & 0.000 & 0.8945 & 0.000 & 0.8945 & 0.000 & 0.8945 & 0.000 \\
		dresses-sales,OpenML-CC18 & 0.5689 & 0.156 & 0.4855 & 0.810 & 0.5524 & 0.012 & 0.5520 & 0.000 \\
		electricity,OpenML-CC18 & 0.7621 & 0.005 & -- & -- & 0.7574 & 0.273 & 0.7621 & 0.000 \\
		eucalyptus,OpenML-CC18 & 0.5539 & 0.224 & 0.6044 & 0.353 & 0.5912 & 0.481 & 0.5925 & 0.000 \\
		first-order-theorem-proving,OpenML-CC18 & 0.4485 & 0.075 & 0.4505 & 0.219 & 0.4379 & 0.168 & 0.4580 & 0.000 \\
		har,OpenML-CC18 & 0.8764 & 0.006 & 0.8885 & 0.050 & 0.9100 & 0.213 & 0.8749 & 0.000 \\
		ilpd,OpenML-CC18 & 0.7015 & 0.088 & 0.7635 & 0.844 & 0.6703 & 0.134 & 0.6964 & 0.000 \\
		isolet,OpenML-CC18 & 0.4049 & 0.062 & 0.3862 & 0.088 & 0.4147 & 0.215 & 0.3991 & 0.000 \\
		jm1,OpenML-CC18 & 0.8050 & 0.165 & 0.8123 & 0.720 & 0.8245 & 0.326 & 0.8075 & 0.000 \\
		jungle\_chess\_2pcs\_raw\_endgame\_complete,OpenML-CC18 & 0.7125 & 0.526 & -- & -- & 0.7226 & 0.091 & 0.7182 & 0.000 \\
		kc1,OpenML-CC18 & 0.8588 & 0.097 & 0.8710 & 0.589 & 0.8531 & 0.158 & 0.8497 & 0.000 \\
		kc2,OpenML-CC18 & 0.8179 & 0.211 & 0.8181 & 0.724 & 0.8204 & 0.038 & 0.8238 & 0.000 \\
		kr-vs-kp & 0.9409 & 0.000 & 0.9409 & 0.000 & 0.9409 & 0.000 & 0.9409 & 0.000 \\
		letter & 0.2895 & 0.302 & 0.2497 & 0.042 & 0.2522 & 0.012 & 0.2546 & 0.000 \\
		madelon,OpenML-CC18 & 0.7371 & 0.026 & 0.7643 & 0.427 & 0.7459 & 0.144 & 0.7292 & 0.000 \\
		mfeat-factors & 0.5111 & 0.070 & 0.5290 & 0.099 & 0.5287 & 0.138 & 0.5175 & 0.000 \\
		mfeat-fourier & 0.5534 & 0.075 & 0.5625 & 0.102 & 0.5666 & 0.160 & 0.5400 & 0.000 \\
		mfeat-karhunen & 0.6151 & 0.105 & 0.6721 & 0.243 & 0.6513 & 0.208 & 0.5955 & 0.000 \\
		mfeat-morphological & 0.7805 & 0.478 & 0.6709 & 0.100 & 0.7538 & 0.488 & 0.6540 & 0.000 \\
		mfeat-pixel,OpenML-CC18 & 0.8160 & 0.332 & 0.8132 & 0.339 & 0.8258 & 0.405 & 0.7700 & 0.000 \\
		mfeat-zernike & 0.5047 & 0.052 & 0.5172 & 0.087 & 0.5488 & 0.149 & 0.4975 & 0.000 \\
		mfeat-zernike,OpenML-CC18 & 0.5503 & 0.134 & 0.5449 & 0.124 & 0.5629 & 0.165 & 0.5150 & 0.000 \\
		mnist\_784,OpenML-CC18 & 0.5728 & 0.003 & -- & -- & 0.5132 & 0.907 & 0.5721 & 0.000 \\
		nomao,OpenML-CC18 & 0.9423 & 0.076 & -- & -- & 0.9435 & 0.138 & 0.9291 & 0.000 \\
		numerai28.6,OpenML-CC18 & 0.5167 & 0.003 & -- & -- & 0.5173 & 0.162 & 0.5166 & 0.000 \\
		optdigits,OpenML-CC18 & 0.5343 & 0.071 & 0.5370 & 0.088 & 0.5426 & 0.168 & 0.5190 & 0.000 \\
		ozone-level-8hr,OpenML-CC18 & 0.9376 & 0.076 & 0.9652 & 0.514 & 0.9341 & 0.054 & 0.9313 & 0.000 \\
		pc1,OpenML-CC18 & 0.9407 & 0.105 & 0.7163 & 0.639 & 0.9319 & 0.006 & 0.9287 & 0.000 \\
		pc3,OpenML-CC18 & 0.9032 & 0.088 & 0.9309 & 0.519 & 0.8964 & 0.025 & 0.8899 & 0.000 \\
		pc4,OpenML-CC18 & 0.9064 & 0.125 & 0.7792 & 0.907 & 0.8882 & 0.006 & 0.8896 & 0.000 \\
		pendigits,OpenML-CC18 & 0.7059 & 0.140 & 0.7175 & 0.141 & 0.7312 & 0.342 & 0.7256 & 0.000 \\
		phoneme,OpenML-CC18 & 0.7942 & 0.014 & 0.8108 & 0.246 & 0.8047 & 0.188 & 0.7929 & 0.000 \\
		qsar-biodeg,OpenML-CC18 & 0.8317 & 0.218 & 0.8532 & 0.366 & 0.8425 & 0.167 & 0.8038 & 0.000 \\
		satimage,OpenML-CC18 & 0.8044 & 0.078 & 0.8247 & 0.233 & 0.8739 & 0.256 & 0.7883 & 0.000 \\
		segment,OpenML-CC18 & 0.7710 & 0.087 & 0.7700 & 0.038 & 0.7892 & 0.103 & 0.7563 & 0.000 \\
		semeion,OpenML-CC18 & 0.5775 & 0.000 & 0.5775 & 0.000 & 0.5775 & 0.000 & 0.5769 & 0.000 \\
		sick,OpenML-CC18 & 0.9908 & 0.084 & 0.9887 & 0.365 & 0.9848 & 0.003 & 0.9846 & 0.000 \\
		spambase,OpenML-CC18 & 0.8962 & 0.020 & 0.8880 & 0.399 & 0.8899 & 0.407 & 0.8937 & 0.000 \\
		splice,OpenML-CC18 & 0.9000 & 0.253 & 0.9040 & 0.087 & 0.9122 & 0.000 & 0.9122 & 0.000 \\
		steel-plates-fault,OpenML-CC18 & 0.5910 & 0.071 & 0.6054 & 0.111 & 0.6101 & 0.144 & 0.5940 & 0.000 \\
		texture,OpenML-CC18 & 0.5381 & 0.017 & 0.5412 & 0.057 & 0.5734 & 0.166 & 0.5362 & 0.000 \\
		tic-tac-toe,OpenML-CC18 & 0.9080 & 0.777 & 0.7474 & 0.000 & 0.7474 & 0.000 & 0.7474 & 0.000 \\
		vehicle,OpenML-CC18 & 0.6838 & 0.209 & 0.7332 & 0.273 & 0.6954 & 0.103 & 0.6821 & 0.000 \\
		vowel,OpenML-CC18 & 0.4883 & 0.169 & 0.5164 & 0.478 & 0.5572 & 0.327 & 0.4687 & 0.000 \\
		wall-robot-navigation,OpenML-CC18 & 0.9094 & 0.089 & 0.9061 & 0.146 & 0.9245 & 0.081 & 0.9146 & 0.000 \\
		wdbc,OpenML-CC18 & 0.9452 & 0.134 & 0.9517 & 0.195 & 0.9411 & 0.100 & 0.9227 & 0.000 \\
		wilt,OpenML-CC18 & 0.9818 & 0.014 & 0.9885 & 0.299 & 0.9831 & 0.083 & 0.9775 & 0.000 \\
		\bottomrule
	\end{tabular}
\end{table}

\section{Full Combination Results on OpenML-CC18}
\label{app:cc18_full}

Table~\ref{tab:app_cc18_full} presents decided accuracy for 
all 11 combinations averaged across the 71 CC18 datasets,
supporting the routing comparison in Section~\ref{sec:cc18}.

\begin{table}[h]
	\caption{All 11 combinations on OpenML-CC18.
		Mean decided accuracy and boundary-uncertain rate across
		71 datasets (5-fold CV). prob. = probabilistic routing,
		h.m. = hard-middle routing.}
	\label{tab:app_cc18_full}
	\centering
	\small
	\begin{tabular}{llcc}
		\toprule
		Method & Routing & Dec.Acc & Undec\% \\
		\midrule
		Class-Overlap & prob. & 0.7624 & 0.712 \\
		Gain-Ratio & prob. & 0.7562 & 0.584 \\
		Node-Bootstrap & h.m. & 0.7458 & 0.177 \\
		Node-Bootstrap & prob. & 0.7455 & 0.323 \\
		Margin & prob. & 0.7426 & 0.168 \\
		Quality-Plateau & prob. & 0.7361 & 0.174 \\
		Baseline & ref. & 0.7247 & 0.000 \\
		Margin & h.m. & 0.7229 & 0.097 \\
		Quality-Plateau & h.m. & 0.7148 & 0.109 \\
		Gain-Ratio & h.m. & 0.7032 & 0.269 \\
		Class-Overlap & h.m. & 0.6788 & 0.323 \\
		\bottomrule
	\end{tabular}
\end{table}

\section{Breiman Synthetic Benchmarks: All Delta Methods}
\label{app:breiman_full}

Table~\ref{tab:app_breiman} extends Table~\ref{tab:breiman} 
in the main paper to include all five delta methods, providing
a complete view of the U/B ratio diagnostic.

\begin{table}[h]
	\caption{All delta methods on Breiman synthetic benchmarks
		under probabilistic routing. U/B = boundary-uncertain rate 
		divided by known Bayes error.}
	\label{tab:app_breiman}
	\centering
	\small
	\begin{tabular}{lcccccc}
		\toprule
		& \multicolumn{2}{c}{waveform (BE=14\%)}
		& \multicolumn{2}{c}{twonorm (BE=2.3\%)}
		& \multicolumn{2}{c}{ringnorm (BE=1.7\%)} \\
		\cmidrule(lr){2-3}\cmidrule(lr){4-5}\cmidrule(lr){6-7}
		Method & Dec.Acc & U/B & Dec.Acc & U/B & Dec.Acc & U/B \\
		\midrule
		Baseline & 0.7350 & 0.00 & 0.7759 & 0.00 & 0.7647 & 0.00 \\
		Class-Overlap & 0.8578 & 6.87 & 0.9483 & 39.05 & 0.8431 & 42.71 \\
		Gain-Ratio & 0.8492 & 5.95 & 0.9233 & 37.46 & 0.8261 & 46.22 \\
		Margin & 0.7377 & 0.09 & 0.7774 & 0.42 & 0.7648 & 0.39 \\
		Node-Bootstrap & 0.7845 & 2.57 & 0.8245 & 15.64 & 0.7688 & 11.76 \\
		Quality-Plateau & 0.7711 & 2.54 & 0.8079 & 13.49 & 0.7879 & 10.68 \\
		\bottomrule
	\end{tabular}
\end{table}

\section{Medical and Financial Datasets: All Combinations}
\label{app:medical_full}

Table~\ref{tab:app_medical} extends Table~\ref{tab:medical}
to include all 11 combinations on the four medical and 
financial datasets, supporting the domain analysis in
Section~\ref{sec:medical}.

\begin{table}[h]
	\caption{All combinations on medical and financial datasets
		(5-fold CV means). Dec.Acc = decided accuracy.
		Undec\% = boundary-uncertain rate.}
	\label{tab:app_medical}
	\centering
	\small
	\begin{tabular}{llcccc}
		\toprule
		Dataset & Method & Routing & Dec.Acc & Undec\% & Acc.All \\
		\midrule
		credit-g & Node-Bootstrap & prob. & 0.7677 & 0.710 & 0.7180 \\
		& Margin & h.m. & 0.7631 & 0.271 & 0.7150 \\
		& Gain-Ratio & prob. & 0.7287 & 0.271 & 0.7190 \\
		& Node-Bootstrap & h.m. & 0.7161 & 0.323 & 0.6940 \\
		& Baseline & ref. & 0.7130 & 0.000 & 0.7130 \\
		& Quality-Plateau & prob. & 0.7126 & 0.077 & 0.7130 \\
		& Quality-Plateau & h.m. & 0.7089 & 0.049 & 0.7080 \\
		& Gain-Ratio & h.m. & 0.7035 & 0.162 & 0.7030 \\
		& Class-Overlap & prob. & 0.6993 & 0.901 & 0.7210 \\
		& Class-Overlap & h.m. & 0.6983 & 0.481 & 0.6960 \\
		& Margin & prob. & 0.6808 & 0.481 & 0.7130 \\
		diabetes & Gain-Ratio & prob. & 0.8909 & 0.797 & 0.7473 \\
		& Node-Bootstrap & prob. & 0.8363 & 0.634 & 0.7499 \\
		& Class-Overlap & h.m. & 0.8356 & 0.542 & 0.7239 \\
		& Class-Overlap & prob. & 0.8200 & 0.948 & 0.7500 \\
		& Gain-Ratio & h.m. & 0.8170 & 0.617 & 0.7513 \\
		& Node-Bootstrap & h.m. & 0.8059 & 0.491 & 0.7369 \\
		& Quality-Plateau & prob. & 0.7772 & 0.294 & 0.7408 \\
		& Quality-Plateau & h.m. & 0.7766 & 0.285 & 0.7291 \\
		& Margin & prob. & 0.7418 & 0.095 & 0.7343 \\
		& Margin & h.m. & 0.7410 & 0.083 & 0.7317 \\
		& Baseline & ref. & 0.7382 & 0.000 & 0.7382 \\
		heart-c & Class-Overlap & prob. & 0.7963 & 0.744 & 0.7407 \\
		& Margin & prob. & 0.7744 & 0.589 & 0.7296 \\
		& Gain-Ratio & prob. & 0.7408 & 0.311 & 0.7407 \\
		& Baseline & ref. & 0.7370 & 0.000 & 0.7370 \\
		& Gain-Ratio & h.m. & 0.7341 & 0.211 & 0.7370 \\
		& Quality-Plateau & prob. & 0.7296 & 0.000 & 0.7296 \\
		& Quality-Plateau & h.m. & 0.7222 & 0.000 & 0.7222 \\
		& Node-Bootstrap & prob. & 0.7102 & 0.607 & 0.7407 \\
		& Node-Bootstrap & h.m. & 0.6471 & 0.093 & 0.6556 \\
		& Class-Overlap & h.m. & 0.5886 & 0.104 & 0.6037 \\
		& Margin & h.m. & 0.5869 & 0.004 & 0.5889 \\
		mammography & Gain-Ratio & prob. & 0.9950 & 0.196 & 0.9838 \\
		& Class-Overlap & prob. & 0.9948 & 0.359 & 0.9846 \\
		& Gain-Ratio & h.m. & 0.9944 & 0.226 & 0.9835 \\
		& Node-Bootstrap & prob. & 0.9915 & 0.108 & 0.9848 \\
		& Node-Bootstrap & h.m. & 0.9911 & 0.091 & 0.9847 \\
		& Class-Overlap & h.m. & 0.9881 & 0.033 & 0.9838 \\
		& Quality-Plateau & prob. & 0.9880 & 0.055 & 0.9847 \\
		& Quality-Plateau & h.m. & 0.9879 & 0.025 & 0.9854 \\
		& Margin & h.m. & 0.9867 & 0.005 & 0.9852 \\
		& Margin & prob. & 0.9867 & 0.007 & 0.9845 \\
		& Baseline & ref. & 0.9844 & 0.000 & 0.9844 \\
		\bottomrule
	\end{tabular}
\end{table}

\section{Reproducibility}
\label{app:repro}

All experiments were implemented in Python~3.12 using 
scikit-learn~1.3~\cite{pedregosa2011scikit}, numpy~1.24, 
scipy~1.10, and openml~0.14~\cite{vanschoren2014openml}.
Experiments were run on AMD Ryzen 7 9700X 8-Core Processor, 32GB RAM.
Total benchmark runtime for the CC18 collection was 
approximately CC18: 16933s (~4h 42m) Breiman: 346s (~0h 5m) Medical: 29s.
All random seeds were fixed at 42.
Code is available at \url{https://github.com/smitswil/ternary_tree}.

\section{Ternary Tree Structure Example}
\label{app:tree_structure}

Figure~\ref{fig:tree_structure} shows an annotated ternary decision
tree trained on the breast cancer dataset using the margin $\delta$
method at depth~2 (TrinaryTree, hard-middle routing).
Each split node carries three outgoing branches corresponding to the
three zones of Equation~\ref{eq:zone}: a decisive left branch
($\leq \theta - \delta$), a decisive right branch
($> \theta + \delta$), and a physical middle branch for
boundary-uncertain instances (dashed orange edges).
The $\delta$ value at each node (shown in red) varies with the local
data geometry, demonstrating the per-node adaptivity of the method.
Nodes and leaves reached via the middle branch are shaded orange to
make the boundary-uncertain path traceable from root to leaf.

\begin{figure}[H]
\centering
\includegraphics[width=\textwidth]{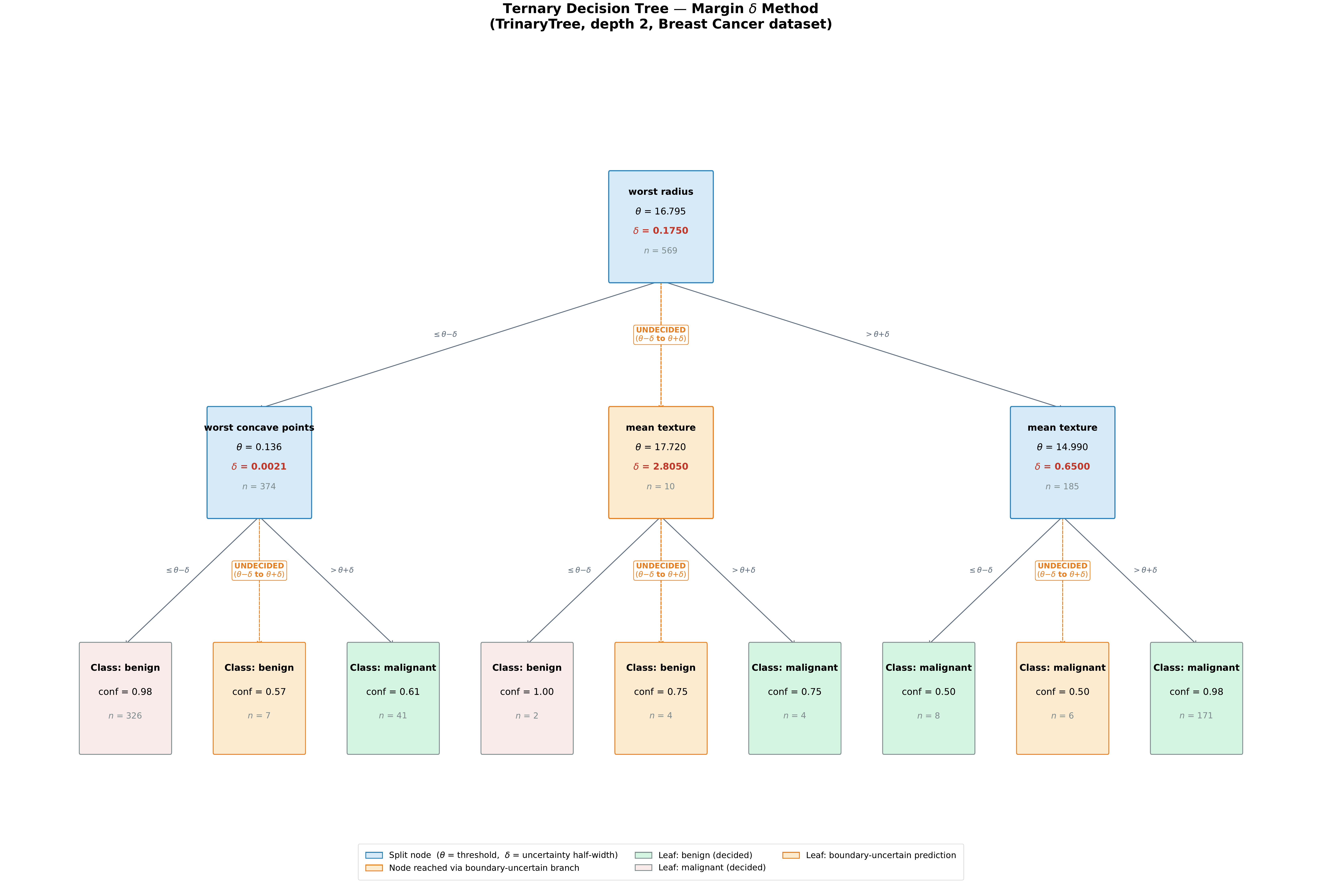}
\caption{Ternary decision tree structure (hard-middle routing,
margin $\delta$ method, depth~2, Breast Cancer dataset).
Each split node routes instances decisively left
($\leq \theta - \delta$), decisively right ($> \theta + \delta$),
or to a dedicated middle subtree for boundary-uncertain instances
(dashed orange edges).
The uncertainty zone half-width $\delta$ (red) is computed locally
at each node from the margin to the nearest cross-class training
example.
In probabilistic routing (Section~\ref{sec:routing}), the same
three-zone evaluation occurs at prediction time without a physical
middle branch, blending child subtree outputs by distance to
$\theta$ instead.}
\label{fig:tree_structure}
\end{figure}

\section{Probabilistic Routing Tree Structure Example}
\label{app:prob_tree}

Figure~\ref{fig:prob_tree} shows a BinaryTernaryTree trained on the same
breast cancer dataset as Figure~\ref{fig:tree_structure}, using the same
margin $\delta$ method and depth, for direct comparison between the two
routing architectures.

In probabilistic routing the tree structure is identical to standard CART:
every split node has exactly two physical children.
The uncertainty zone $[\theta - \delta,\; \theta + \delta]$ is not a
separate branch but a routing rule applied at prediction time.
Instances whose feature value falls within this range receive a prediction
formed by distance-weighted blending of both child subtree outputs:
\begin{equation*}
\hat{p} = w_L \hat{p}_L + w_R \hat{p}_R, \quad
w_L = \frac{\theta + \delta - x_f}{2\delta}, \quad w_R = 1 - w_L.
\end{equation*}
Each node box annotates the blend range $[\theta{-}\delta,\, \theta{+}\delta]$
(grey italic) and the routing logic is summarised in the annotation box.
Comparing this figure to Figure~\ref{fig:tree_structure} (Appendix~\ref{app:tree_structure})
illustrates the architectural distinction: hard-middle routing creates a physical
third subtree trained only on uncertain instances, whereas probabilistic routing
reuses both existing subtrees and blends their outputs at inference time.

\begin{figure}[H]
\centering
\includegraphics[width=\textwidth]{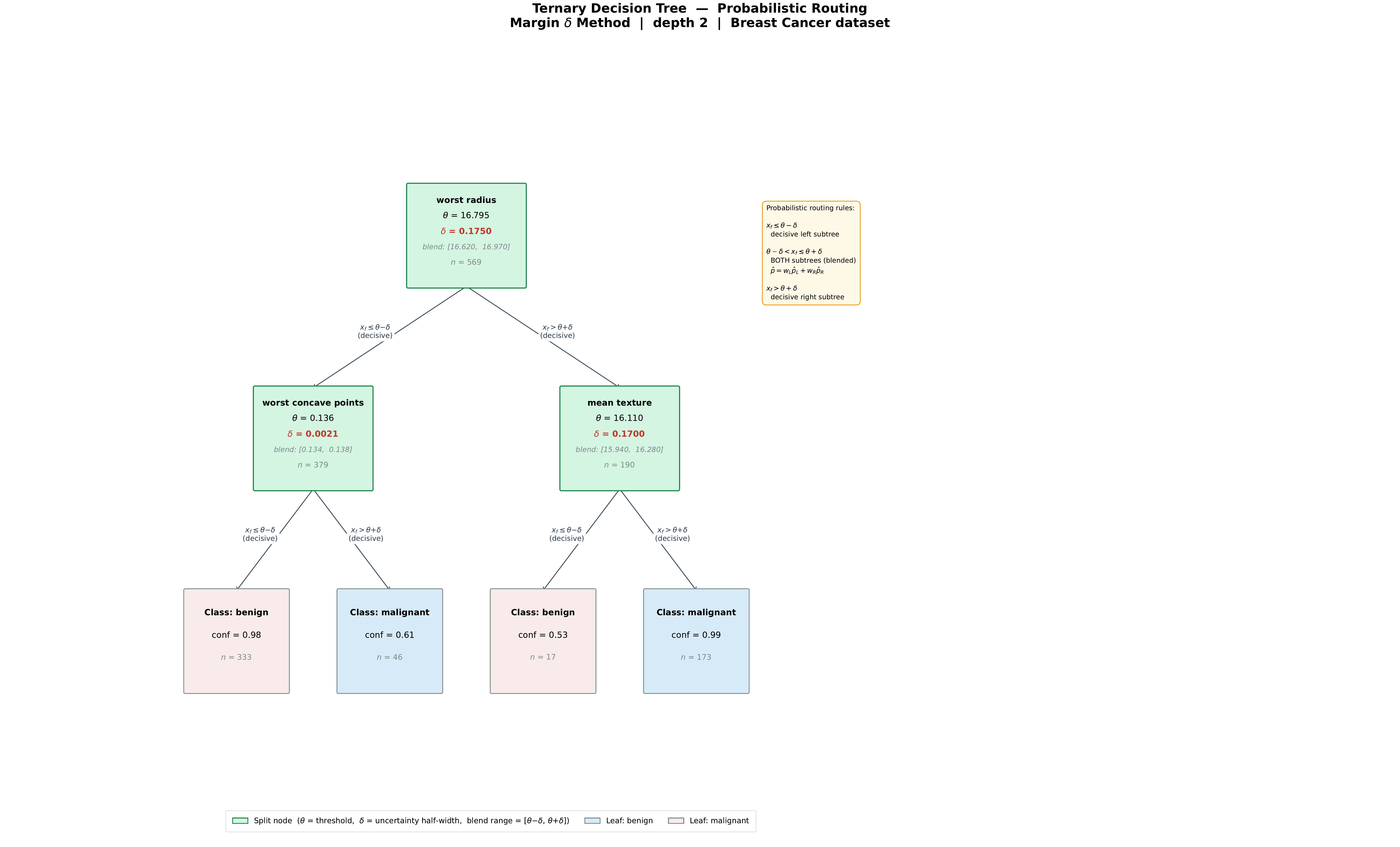}
\caption{BinaryTernaryTree structure (probabilistic routing, margin $\delta$
method, depth~2, Breast Cancer dataset).
Each split node has two physical children, identical in structure to standard CART.
The uncertainty zone $[\theta{-}\delta,\; \theta{+}\delta]$ (blend range, shown in
grey italic) is not a physical branch: instances in this zone receive predictions
formed by distance-weighted blending of both child subtree outputs simultaneously,
with weights proportional to distance from $\theta$.
The routing annotation box (upper right) summarises the three-zone logic.
Compare with Figure~\ref{fig:tree_structure} (hard-middle routing) where the
same zone corresponds to a dedicated physical middle subtree.}
\label{fig:prob_tree}
\end{figure}

\end{document}